\documentclass{article}

\usepackage{arxiv}
\usepackage{natbib}
\usepackage{algorithm}

\usepackage{bm}
\usepackage{bbm}
\usepackage{mathtools} %
\usepackage{amsmath,amsfonts,amssymb}
\usepackage{amsthm}
\usepackage{algorithmic}
\usepackage{hyperref}
\usepackage[capitalise]{cleveref}
\usepackage{balance}

\newcommand{\ba}{\boldsymbol{a}}
\newcommand{\bb}{\boldsymbol{b}}

\newcommand{\be}{\boldsymbol{e}}

\newcommand{\bp}{\boldsymbol{p}}
\newcommand{\bg}{\boldsymbol{g}}
\newcommand{\bx}{\boldsymbol{x}}
\newcommand{\bu}{\boldsymbol{u}}
\newcommand{\by}{\boldsymbol{y}}

\newcommand{\bz}{\boldsymbol{z}}

\newcommand{\bw}{\boldsymbol{w}}

\newcommand{\argmin}{\mathop{\mathrm{argmin}}}

\newcommand{\interior}{\mathop{\mathrm{int}}}
\newcommand{\dom}{\mathop{\mathrm{dom}}}

\newcommand{\field}[1]{\mathbb{#1}}

\newcommand{\R}{\field{R}}

\newtheorem{theorem}{Theorem}[section]
\newtheorem{lemma}{Lemma}[section]

\newtheorem{corollary}{Corollary}[theorem]
\crefname{prop}{proposition}{propositions}
\crefname{lemma}{lemma}{lemmas}
\allowdisplaybreaks
\usepackage{thm-restate}

\title{A Closer Look at Temporal Variability in Dynamic Online Learning}

\author{
  Nicolò Campolongo\\
  Università di Milano\\
  \texttt{nicolo.campolongo@unimi.it} \\
   \And
  Francesco Orabona \\
  Boston University\\
  \texttt{francesco@orabona.com} \\
}

\begin{document}

\twocolumn[\maketitle]

\begin{abstract}
    This work focuses on the setting of dynamic regret in the context of online learning with full information. In particular, we analyze regret bounds with respect to the temporal variability of the loss functions. By assuming that the sequence of loss functions does not vary much with  time, we show that it is possible to incur improved regret bounds compared to existing results. The key to our approach is to use the loss function (and not its gradient) during the optimization process. Building on recent advances in the analysis of Implicit algorithms, we propose an adaptation of the Implicit version of Online Mirror Descent to the dynamic setting. Our proposed algorithm is adaptive not only to the temporal variability of the loss functions, but also to the path length of the sequence of comparators when an upper bound is known. Furthermore, our analysis reveals that our results are tight and cannot be improved without further assumptions. Next, we show how our algorithm can be applied to the setting of learning with expert advice or to settings with composite loss functions. Finally, when an upper bound to the path-length is not fixed beforehand we show how to combine a greedy strategy with existing strongly-adaptive algorithms to compete optimally against different sequences of comparators simultaneously.
\end{abstract}

\section{Introduction}

Online learning is a powerful tool in modeling many practical scenarios. Furthermore, in recent years it has led to advancements in various areas of machine learning in general, both practically and theoretically. 
Formally, given a convex set $\mathcal{V} \subseteq \mathbb{R}^d$, a time horizon $T$ and a sequence of cost functions $\ell_1, \ldots, \ell_t $, in the online learning setting the goal is to design algorithms such that for any comparator model $\bu \in \mathcal{V}$ the regret is minimized,
\begin{equation*}
    R_T(\bu) \triangleq \sum_{t=1}^T \ell_t(\bx_t) - \sum_{t=1}^T \ell_t(\bu),
\end{equation*}
where $\bx_t$ is the output of the algorithm at time $t$.
In particular, the objective is to have algorithms whose regret can be provably upper bounded by a quantity which grows sublinearly in $T$.

While the static regret is a well-studied objective and many algorithms have a sublinear regret upper bound, sometimes competing with the best comparator is not meaningful. Indeed, there are situations where the environment is not stationary. In this case, rather than comparing the performance of an algorithm against a single fixed model, it is preferable to compete against a ``moving'' target, i.e., a sequence of different comparators. In this work, we focus on online learning in the dynamic setting, considering the \emph{full-information} feedback, where in every round the loss function is revealed. 

To model dynamic environments, stronger notions of regret are used. In particular, we consider the \emph{general dynamic regret} \citep{zinkevich2003online,hall2013dynamical} against the sequence $\bu_{1:T} \triangleq (\bu_1, \ldots, \bu_T)$ as
\begin{equation} \label{eq:dynamic_regret}
    R_T(\bu_{1:T}) \triangleq \sum_{t=1}^T \ell_t(\bx_t) - \sum_{t=1}^T \ell_t(\bu_t)~.
\end{equation} 
It can be shown that it is impossible to achieve sublinear dynamic regret in the worst-case. However, if one puts some restrictions on the sequence $\bu_{1:T}$ and makes some regularities assumptions, then \cref{eq:dynamic_regret} can be sublinear in $T$. There are various measures which can be used to model the regularity of the environment. A natural measure of non-stationarity introduced in \citet{zinkevich2003online} is the \emph{path-length}\footnote{One could also consider other versions of path-length, such as its squared version \citep{yang2016tracking}.} of the sequence $\bu_{1:T}$, which we denote by
\begin{equation} \label{eq:path_length}
    C_T(\bu_{1:T}) \triangleq \sum_{t=2}^T \| \bu_t - \bu_{t-1} \|~.
\end{equation}
Another measure of non-stationarity is given by the \emph{temporal variability} of the loss functions~\citep{besbes2015non}. Formally, let $\ell_{1:T}$ be the shorthand for $(\ell_1, \ldots, \ell_T)$, the temporal variability of a sequence $\ell_{1:T}$ is defined as 
\begin{equation} \label{eq:temporal_variab}
    V_T(\ell_{1:T}) \triangleq \sum_{t=2}^T \max_{\bx \in \mathcal{V}} | \ell_t(\bx) - \ell_{t-1}(\bx) |~.
\end{equation}
In the remaining we will use the shorthands $C_T$ for $C_T(\bu_{1:T})$ and $V_T$ for $V_T(\ell_{1:T})$ when the context is clear.
A particular case of dynamic regret is the so-called \emph{restricted} setting~\citep{besbes2015non,jadbabaie2015online,yang2016tracking}. In this setting, the sequence of comparators is given by the local minimizers of the loss functions, i.e., $\bu_{1:T}^* := (\bu_1^*, \ldots, \bu_T^*)$, where $\bu_t^* = \argmin_{\bx \in \mathcal{V}} \ell_t(\bx)$. 

Most recent developments in online learning have been driven by the use of two paradigms: \emph{Online Mirror Descent} (OMD) and \emph{Follow The Regularized Leader} (FTRL) (see the surveys \citet{shalev2012online,Orabona19}). Both of them usually achieve the same regret bounds thanks to the linearization trick: given the convexity of the loss functions one can exploit the fact that $\ell_t(\bx_t) - \ell_t(\bu) \leq \langle \bg_t, \bx_t - \bu \rangle$, where $\bg_t \in \partial \ell_t(\bx_t)$ is a subgradient of the loss function. One can therefore shift her goal to minimize this new objective over time. 
On the other hand, we choose to not use subgradients in the optimization process but the loss function directly. 
We will show that this is the key factor in order to obtain dynamic regret bounds depending on the temporal variability $V_T$.

\paragraph{Contributions.}
The main results of this paper are summarized below:
\begin{itemize}
    \item In \cref{sec:temp_var}, we show that there exists a simple strategy which achieves an upper bound of $\mathcal{O}(V_T)$ on the dynamic regret. We also provide a lower bound which shows that this regret bound is tight. Next, we show when this strategy fails and why we need different algorithms.
    \item In \cref{sec:dynamic_iomd}, using recent advances in the analysis of implicit updates in online learning we design an algorithm which is adaptive to both $C_T$ and $V_T$. Using an adaptation of OMD to the implicit case, we will provide an algorithm which incurs a dynamic regret bound of $\mathcal{O}(V_T, \sqrt{T(1+\tau)})$, for all sequence of comparators whose path-length $C_T$ is upper bounded by $\tau$.
    \item Finally, when the complexity of the class of comparators is not fixed in advance in terms of path-length (i.e., an upper bound $C_T \le \tau$ is not fixed beforehand), in \cref{sec:combiner} we show how to combine the strategy from \cref{sec:temp_var} with an existing algorithm and get the optimal bound of $\tilde{\mathcal{O}}(\min\{V_T, \sqrt{T(1 + C_T(\bu_{1:T}))}\})$\footnote{The $\tilde{\mathcal{O}}$ notation hides poly-logarithmic terms.} for any possible sequence $\bu_{1:T}$. 
\end{itemize}

\section{Related work}

In this section, we are going to review the two lines of work most related to ours: algorithms designed for non-stationary environments and implicit updates in online learning. We recap existing results and highlight both similarities and differences compared to our results.

\vspace{-0.1cm}

\paragraph{Path-length.} The notion of dynamic regret was first introduced in the seminal work of \citet{zinkevich2003online}, which proved that \emph{Online Gradient Descent} incurs a regret bound of $\mathcal{O}(\sqrt{T}(1+ C_T))$. This result was later extended by \citet{hall2013dynamical} who considered a modified (and possibly richer) definition of path-length. A lower bound of $\Omega( \sqrt{T(1 + C_T)})$\footnote{To avoid clutter, w.l.o.g. we supress parameters other than $T$ and $C_T$ in the asymptotic notation.} in terms of path-length is shown in \citet[Theorem 2]{zhang2018adaptive}, who also provide an algorithm which matches it. 

\vspace{-0.1cm}

\paragraph{Temporal Variability.} \citet{besbes2015non} provided an analysis of restarted gradient descent in the setting of stochastic optimization with noisy gradients which incurs $\mathcal{O}(T^{2/3} (V_T' + 1)^{1/3} ) $, where $V_T'$ is an upper bound on $V_T$ known in advance. \citet{jadbabaie2015online} gave an algorithm achieving a restricted dynamic regret of $\tilde{\mathcal{O}}( \sqrt{G_T} + \min(\sqrt{(G_T+1)C_T}, ((G_T+1)T)^{1/3} (V_T+1)^{2/3}))$, where $G_T = \sum_{t=1}^T \| \nabla f_t(\bx_t) - \bp_t \|_\star^2$ and $\bp_1, \ldots, \bp_T$ is a predictable sequence computable at the start of round $t$. Importantly, this bound is obtained without prior knowledge of $G_T$, $C_T$ and $V_T$ but under the assumption that all of them can be \emph{observable}. If one limits the algorithm to not use predictable sequences, then the bound reduces to $\tilde{\mathcal{O}}(\sqrt{T} + \min\{\sqrt{T(1 + C_T)}, T^{1/3} (V_T+1)^{2/3})\}$. In \cref{sec:dynamic_iomd}, we design an algorithm similar in spirit to the one from \citet{jadbabaie2015online}, which incurs an improved regret bound of $\min\{\sqrt{T(1+C_T)}, V_T \}$ when $C_T$ is fixed in advance or can be observed.

\vspace{-0.05cm}

\paragraph{Adaptive Regret.} A parallel line of work on non-stationary environments involves the study of the weakly and strongly-adaptive regret~\citep{hazan2007adaptive,daniely2015strongly}, which aims to minimize the static regret over any possible (sub)interval over the time horizon $T$. Importantly, it has been shown that strongly-adaptive regret bounds imply dynamic regret bounds. Recently, \citet{cutkosky2020parameter} provided a strongly-adaptive algorithm that achieves the optimal dynamic regret bound in terms of path-length, for any sequence of comparators.
On the other hand, in \citet{zhang2018dynamic} it is shown that dynamic regret of the strongly adaptive algorithm given in \citet{jun2017improved} is $\tilde{\mathcal{O}}(T^{2/3} (V_T+1)^{1/3})$.
To summarize, using existing strongly-adaptive algorithms, a dynamic regret bound of $ \tilde{\mathcal{O}}(\min\{ \sqrt{T(1+C_T)}, T^{2/3} (V_T+1)^{1/3}\}) $ for any sequence $\bu_{1:T}$ can be achieved, without knowing its path-length in advance. In \cref{sec:combiner}, we show instead how to improve this bound to $\mathcal{O}(\min\{\sqrt{T(1+C_T)}, V_T \}$.

\vspace{-0.1cm}

\paragraph{Implicit Algorithms.} Implicit algorithms are known in the optimization literature as proximal methods \citep{parikh2014proximal} and can be traced back to the work of \citet{moreau1965proximite}.  In the online learning community, they have been introduced in \citet{kivinen1997exponentiated}. 
In a recent work \citet{campolongo2020temporal} showed that implicit updates can outperform their linearized counterparts when the temporal variability is low, in the static setting.
The dynamic regret for proximal algorithms in the online setting has been also studied in the case of strongly convex losses in \citet{dixit2019online} and for composite losses in \citet{ajalloeian2020inexact}, but both these works assume different notions of feedback from ours.

\section{Definitions}

For a function $f:\R^d \rightarrow (-\infty, +\infty]$, we define a \emph{subgradient} of $f$ in $\bx \in \R^d$ as a vector $\bg \in \R^d$ that satisfies $f(\by)\geq f(\bx) + \langle \bg, \by-\bx\rangle, \ \forall \by \in \R^d$. 
We denote the set of subgradients of $f$ in $\bx$ by $\partial f(\bx)$. We denote by $\be_i$ the standard basis vectors, for $i=1,\dots,d$. We denote the expected value of a random variable $\bx$ by $\mathbb{E}[\bx]$ and the indicator function of the event $A$ by $\mathbbm{1}\{A\}$.
We denote the \emph{dual norm} of $\|\cdot\|$ by $\|\cdot\|_\star$.
A proper function $f : \R^d \rightarrow (-\infty, +\infty]$ is \emph{$\mu$-strongly convex} over a convex set $V \subseteq \interior \dom f$ w.r.t. $\|\cdot\|$ if $\forall \bx, \by \in V$ and $\bg \in \partial f(\bx)$, we have $f(\by) \geq f(\bx) + \langle \bg , \by - \bx \rangle + \frac{\mu}{2} \| \bx - \by \|^2$.
Let $\psi: X \rightarrow \R$ be strictly convex and continuously differentiable on $\interior X$. The \emph{Bregman Divergence} w.r.t. $\psi$ is $B_\psi : X \times \interior X \rightarrow \R_+$ defined as $
B_\psi(\bx, \by) = \psi(\bx) - \psi(\by) - \langle \nabla \psi(\by), \bx - \by\rangle$.
We assume that $\psi$ is strongly convex w.r.t. a norm $\|\cdot\|$ in $\interior X$. We also assume w.l.o.g. the strong convexity constant to be~1, which implies
\begin{equation} \label{eq:bregman_strongly_convex}
B_\psi(\bx,\by)\geq \frac{1}{2}\|\bx-\by\|^2, \quad \forall \bx \in X , \by \in \interior X~.
\end{equation}

\section{A greedy strategy and its limitations}
\label{sec:temp_var}

In this section, we presents a greedy strategy that achieves a dynamic regret bound in terms of temporal variability of the loss functions of $\mathcal{O}(V_T)$. Furthermore, we show that this bound is optimal proving a matching lower bound. Then, we discuss why this strategy might be harmful if the environment is stationary.

The strategy adopted to achieve a bound in terms of temporal variability $V_T$ is depicted in \cref{algo:greedy}. In each round the algorithm plays the minimizer of the observed loss function in the previous round. Despite being reported in \citet{jadbabaie2015online}, we could not find a formal proof of the regret bound of this algorithm. Hence, for completeness we next state a theorem which provides a regret bound to \cref{algo:greedy}.

\begin{algorithm}[t]
\caption{Greedy optimizer}
\label{algo:greedy}
\begin{algorithmic}[1]
{
\REQUIRE{Non-empty closed convex set $\mathcal{V} \subset X \subset \mathbb{R}^d $, $\bx_1 \in \mathcal{V}$}
\FOR{$t=1,\dots,T$}
\STATE{Output $\bx_t \in \mathcal{V}$}
\STATE{Receive $\ell_t: \mathbb{R}^d \rightarrow \R$ and pay $\ell_t(\bx_t)$}
\STATE{Update $\bx_{t+1} = \arg\min_{\bx \in \mathcal{V}} \ell_t(\bx)$}
\ENDFOR
}
\end{algorithmic}
\end{algorithm}

\begin{theorem} \label{thm:greedy_regret}
Let $\mathcal{V} \subset X \subset \mathbb{R}^d $ be non-empty closed convex sets. The regret of \cref{algo:greedy} against any sequence $\bu_{1:T}$ with $\bu_t \in \mathcal{V}$ for all $t$ is bounded as $R_T(\bu_{1:T}) \le \max(V_T, \mathcal{O}(1))$.
\end{theorem}

\begin{proof}
From the update of \cref{algo:greedy}, for any $\bu_t \in \mathcal{V}$ we have that
\begin{align*}
    \sum_{t=1}^T & (\ell_t(\bx_t) - \ell_t(\bu_t)) \\
        & = \sum_{t=1}^T (\ell_t(\bx_t) - \ell_t(\bx_{t+1}) + \ell_t(\bx_{t+1}) - \ell_t(\bu_t)) \\
        & \leq \sum_{t=1}^T (\ell_t(\bx_t) - \ell_t(\bx_{t+1})) \\
        & = \ell_1(\bx_1) - \ell_T(\bx_{T+1}) + \sum_{t=2}^T (\ell_t(\bx_t) - \ell_{t-1}(\bx_t)) \\
        & \leq \ell_1(\bx_1) - \ell_T(\bx_{T+1}) + V_T \\
        & = \max(V_T, \mathcal{O}(1))~. \qedhere
\end{align*} 
\end{proof}
\paragraph{Remark.} The theorem above holds for any sequence of comparators, and in particular for $\bu_{1:T}^*=(\bu_1^*, \dots, \bu_T^*)$ used in the \emph{restricted} setting. At first sight, this result might seem to be in contrast with the result given in \citet{besbes2015non}, which reports a lower bound of $\Omega((V_T^{1/3}+1) T^{2/3})$. However, it should be noted that in \citet{besbes2015non} the feedback is different and not directly comparable to our setting. Indeed, they assume only access to noisy functions and gradients and therefore their Theorem 2 is not applicable.

In \citet{yang2016tracking} it is shown that the same strategy of \cref{algo:greedy} achieves an upper bound of $\mathcal{O}(\max(C_T(\bu_1^*,\dots,\bu_T^*),1))$ when the path-length is taken into account. While the result regarding the path-length is tight, one might wonder if the same could be said about the temporal variability. In the next theorem, we provide a lower bound which shows that the bound in \cref{thm:greedy_regret} is tight.
\begin{theorem} \label{thm:lower_bound_1}
Let $\mathcal{V} = [-1, 1]$, and $C$ be a positive constant independent of $T$. Then, for any algorithm $\mathcal{A}$ on $\mathcal{V}$, and any $\sigma \in (1/\sqrt{T},1)$, there exists a sequence of loss functions $\ell_1,\ldots,\ell_T$ with temporal variability less than or equal to $2 \sigma T $ such that
\begin{equation} \label{eq:lower_bound_restricted}
    R(u_{1:T}) \geq C V_T^\gamma~,
\end{equation}
for any $\gamma \in (0,1)$.
\end{theorem}

\begin{proof}
Similarly to \cite{yang2016tracking}, we consider a simple 1-$d$ problem and employ the following sequence of loss functions. Define $\ell_t(x_t) = \frac12(x_t - \varepsilon_t)^2$, where $\varepsilon_1, \ldots, \varepsilon_T$ is a sequence of random variables sampled uniformly at random between the two values $\{ -\sigma, \sigma\}$. Note that we have $\mathbb{E}[\varepsilon_t] = 0 $ and $\textup{Var}(\epsilon_t) = \mathbb{E}[\varepsilon^2_t] = \sigma^2$. Obviously, the optimal choice in every round is $u_t = \varepsilon_t$. Assume $T\ge 1$. Then, the restricted dynamic regret is given by
\begin{align} \label{eq:lower_1}
    \mathbb{E} \left[ R_T(u_{1:T}) \right] 
    &= \mathbb{E} \left[ \sum_{t=1}^T \ell_t(x_t) - \ell_t(\varepsilon_t) \right] \nonumber \\
    & =\sum_{t=1}^T \frac12 \mathbb{E}[x_t^2] + \frac12 \mathbb{E}[\varepsilon^2] \geq \frac{\sigma^2}{2} T,
\end{align}
where the expectation is taken with respect to the randomness in the sequence of loss functions and any algorithm $\mathcal{A}$, while the inequality is due to the fact that $x_t$ is independent from $\varepsilon_t$ and $\mathbb{E}[\varepsilon_t]=0$. Now, note that we can upper bound the temporal variability as follows
\begin{align} \label{eq:lower_2}
    V_T &= \sum_{t=1}^{T-1} \max_{x \in \mathcal{V}} | \ell_t(x) - \ell_{t+1}(x)| \nonumber \\
        &= \sum_{t=1}^{T-1} \max_{x \in \mathcal{V}} \left| \frac12(x - \varepsilon_t)^2 - \frac12(x-\varepsilon_{t+1})^2 \right| \nonumber \\
        &= \sum_{t=1}^{T-1} \max_{x \in \mathcal{V}} | x(\varepsilon_{t+1} - \varepsilon_t) | \nonumber \\
        &\leq \sum_{t=1}^{T-1} | \varepsilon_{t+1} - \varepsilon_t | \nonumber \\
        &\leq 2 \sigma T~.
\end{align}
Observe that if we set $\sigma = C'/2$ for a positive constant $C'$, then we recover the result in Proposition 1 of \cite{besbes2015non} which says that it is impossible to achieve sublinear dynamic regret unless $V_T = o(T)$.

Now, we let $\sigma = T^{-\mu}$, with $\mu = (1-\gamma)/(2-\gamma)$ and $\mu \in (0,1/2)$. Then, from \cref{eq:lower_1} we have that $R(u_{1:T}) \geq T^{1-2\mu} / 2$, while from \cref{eq:lower_2} we have that $T \ge (V_T/2)^{\frac{1}{1-\mu}}$. Therefore, putting things together we have that $R(u_{1:T}) \ge \frac12 (V_T / 2)^{\frac{1-2\mu}{1-\mu}} = C V_T^\gamma $. Note that if $\gamma=1$ then $\mu=0$ and the regret must be linear in $T$. Therefore, we let $\gamma < 1$.
\end{proof}

\paragraph{Remark.} A lower bound in terms of temporal variability for the \emph{static} regret on constrained domains has been proved in \citet{campolongo2020temporal}, which states that for every $\tau \geq 0$, there exists a sequence of loss functions such that $V_T$ is equal to $\tau$ and the regret satisfies $R_T(\bu) \geq \tau$. However, their proof technique is different and the result is limited to deterministic algorithms, while the lower bound given in \cref{thm:lower_bound_1} holds for randomized algorithms as well.

\cref{eq:lower_bound_restricted} implies that it is impossible to achieve a dynamic regret bound better that $\mathcal{O}(V_T^\gamma)$, with $\gamma < 1$. Note that the proof of lower bound is given for the \emph{restricted} setting, but it automatically implies a lower bound for the general dynamic regret, which includes the \emph{restricted} setting as a particular case.

\paragraph{Greedy fails in the static case.} \cref{algo:greedy} incurs a dynamic regret bounded as $R_T(\bu_{1:T}) \le \mathcal{O}(\min\{V_T, C_T(\bu_{1:T}^*)\})$ for any sequence $\bu_{1:T}$. Based on \cref{thm:lower_bound_1}, this bound is tight when considering $V_T$. However, there are situations where the greedy strategy is doomed to fail.
For example, consider the setting of Learning with Expert Advice with two experts, with $\ell_t(\bx) = \langle \bg_t, \bx \rangle$ and the following choice of $\bg_t$:
\begin{equation*}
    \bg_t = \begin{cases}
        [1, 0], & t \quad \textup{even} \\
        [0, 1], & t \quad \textup{odd}
    \end{cases}
\end{equation*}
It's immediate to see that $C_T(\bu_{1:T}^*) = V_T = T$, and indeed in this case there is no hope of getting an upper bound sublinear in $T$ when considering $\bu_{1:T}^*$. However, if we consider the static case, then the regret of \cref{algo:greedy} against either $\be_1$ or $\be_2$ is $R_T(\be_i) = T - T/2 = \mathcal{O}(T)$, while any algorithm designed for the static setting incurs $R_T(\be_i) \le \mathcal{O}(\sqrt{T})$. Hence, in these situations \cref{algo:greedy} is not be a good choice.

In general, from \cite{zhang2018adaptive} we know that a lower bound of $\Omega(\sqrt{T(1+\tau)})$ holds when considering sequences of comparators whose path-length is upper bounded by $\tau$. A natural question arises: is it possible to keep the rate of $\mathcal{O}(V_T)$ but at the same time to guarantee $\mathcal{O}(\sqrt{T(1+\tau)})$?

To the best of our knowledge there are no algorithms which achieve this goal. Indeed, algorithms designed for dynamic regret such as \citet{jadbabaie2015online} and \citet{besbes2015non} have a regret bound of $\mathcal{O}((V_T^{1/3}+1)T^{2/3})$. The same holds true for strongly adaptive algorithms, as shown in \citet{zhang2018dynamic}. This is not really surprising: a bound of $\mathcal{O}(V_T)$ would imply constant regret in the case the loss functions are fixed, i.e., $\ell_t = \ell$ for all $t$. In this case, using an online-to-batch conversion \citep{cesa2004generalization} would result in a convergence rate of $\mathcal{O}(1/T)$. However, this would be in contrast with the lower bound of $\Omega(1/\sqrt{T})$ by \citet{nesterov2013introductory} on non-smooth batch black-box optimization. Unfortunately, all the algorithms mentioned above make use of gradients and therefore are subject to this lower bound. %

However, not all is lost: in the next section, we illustrate how to achieve a bound of $\mathcal{O}(\min\{V_T, \sqrt{T(1+C_T)}\})$ using an algorithm which makes full use of the loss function (and not just its gradient) for the class of sequences $\bu_{1:T}$ with path-length $C_T$. Moreover, in \cref{sec:combiner} we show how the greedy strategy can be combined with another algorithm in order to achieve the same goal for all sequences $\bu_{1:T}$ simultaneously.

\section{Implicit updates in dynamic environments} \label{sec:dynamic_iomd}

\begin{algorithm}[t]
\caption{Dynamic IOMD}
\label{algo:dynamic_iomd}
\begin{algorithmic}[1]
{
\REQUIRE{Non-empty closed convex set $\mathcal{V} \subset X \subset \mathbb{R}^d $, $\psi: X \rightarrow \mathbb{R}$, $\bx_1 \in \mathcal{V}$, $\gamma$ such that $B_\psi(\bx, \bz) - B_\psi(\by, \bz) \leq \gamma \| \bx - \by \|, \forall \bx, \by, \bz \in \mathcal{V} $, non increasing sequence $(\eta_t)_{t=1}^T$}
\FOR{$t=1,\dots,T$}
\STATE{Output $\bx_t \in \mathcal{V}$}
\STATE{Receive $\ell_t: \mathbb{R}^d \rightarrow \R$ and pay $\ell_t(\bx_t)$}
\STATE{Update $\bx_{t+1} = \arg\min_{\bx \in \mathcal{V}}\ \ell_t(\bx) + B_{\psi}(\bx, \bx_t) / \eta_t  $}
\ENDFOR
}
\end{algorithmic}
\end{algorithm}

In a recent work, \citet{campolongo2020temporal} showed how a modified version of OMD with implicit updates achieves a regret bound in the static setting which is order of $\mathcal{O}(\min(V_T, \sqrt{T}))$. In this section, we show how to adapt this algorithm to satisfy a bound of $\mathcal{O}(\min\{ V_T, \sqrt{T(1+C_T)}\})$ on the dynamic regret when the path-length of the sequence of comparators is fixed to $C_T$.

OMD with implicit updates is depicted in \cref{algo:dynamic_iomd}. The only difference with its linearized counterpart is in the update rule, which uses directly the loss rather than its (sub)gradient in $\bx_t$:
\begin{equation}
\label{eq:implicit_update}
    \bx_{t+1} = \arg\min_{\bx \in \mathcal{V}} \ \ell_t(\bx) + B_{\psi}(\bx, \bx_t) / \eta_t ~.
\end{equation}  
In order to provide a dynamic regret bound to \cref{algo:dynamic_iomd}, we require a Lipschitz continuity condition on the Bregman divergence. Using this assumption, we can get a bound for the dynamic regret shown in the next lemma. We will  use this result in \cref{thm:adaImplicit_regret} to prove the desired bound.
\begin{lemma} \label{thm:dynamic_iomd_bound}
Let $\mathcal{V} \subset X \subset \mathbb{R}^d $ be non-empty closed convex sets, $\psi: X \rightarrow \mathbb{R}$, and $\bx_1 \in \mathcal{V}$. Assume there exists $\gamma \in \mathbb{R}$ such that $B_\psi(\bx, \bz) - B_\psi(\by, \bz) \leq \gamma \| \bx - \by \|, \forall \bx, \by, \bz \in \mathcal{V} $. Define $D^2 \triangleq \max_{\bx, \by \in \mathcal{V}} B_\psi(\bx, \by). $ Let $(\eta_t)_{t=1}^T$ be a non-increasing sequence. Then, the regret of \cref{algo:dynamic_iomd} against any sequence $\bu_{1:T}$ with $\bu_t \in \mathcal{V}$ for all $t$ is bounded as follows
    \begin{equation} 
    \label{eq:base_bound_dynamic}
        R_T(\bu_{1:T}) \leq \frac{D^2}{\eta_T} + \gamma \sum_{t=2}^T \frac{\| \bu_t - \bu_{t-1} \|}{\eta_t} + \sum_{t=1}^T \delta_t,
    \end{equation}
where $\delta_t = \ell_t(\bx_t) - \ell_t(\bx_{t+1}) - B_\psi(\bx_{t+1}, \bx_t) / \eta_t$.
\end{lemma}

\begin{proof}
Let $\bg_t' \in \partial \ell_t(\bx_{t+1})$. From the update rule of \cref{algo:dynamic_iomd} we have that
\begin{align} \label{eq:update_optimality}
    \eta_t &( \ell_t  (\bx_{t+1}) - \ell_t(\bu_t)) \nonumber \\
    & \leq \langle \eta_t \bg_t', \bx_{t+1} - \bu_t \rangle \nonumber \\
    & \leq \langle \nabla \psi(\bx_t) - \nabla \psi(\bx_{t+1}), \bx_{t+1} - \bu_t \rangle \nonumber \\
    & = B_\psi(\bu_t, \bx_t) - B_\psi(\bu_t, \bx_{t+1}) - B_\psi(\bx_{t+1}, \bx_t),
\end{align}
where the first inequality follows from the convexity of the loss functions, while the second from the first-order optimality condition.

Now, we consider the first two terms of the r.h.s. of \cref{eq:update_optimality}. Using the Lipschitz continuity condition on the Bregman divergence and the fact that $\eta_t$ is non-increasing over time, we get
\begin{align*}
    \sum_{t=1}^T & \frac{1}{\eta_t}(B_\psi(\bu_t, \bx_t) - B_\psi(\bu_t, \bx_{t+1})) \\
    & \leq \frac{D^2}{\eta_1} + \sum_{t=2}^T \left( \frac{B_\psi(\bu_t, \bx_t)}{\eta_t} - \frac{B_\psi(\bu_{t-1}, \bx_t)}{\eta_{t-1}} \right) \\
    & = \frac{D^2}{\eta_1} + \sum_{t=2}^T \bigg( \frac{B_\psi(\bu_t, \bx_t)}{\eta_t} - \frac{B_\psi(\bu_{t-1}, \bx_t)}{\eta_t} \\
    & \quad +\frac{B_\psi(\bu_{t-1}, \bx_t)}{\eta_t} - \frac{B_\psi(\bu_{t-1}, \bx_t)}{\eta_{t-1}} \bigg) \\
    & \leq \frac{D^2}{\eta_1} + \gamma \sum_{t=2}^T \frac{\| \bu_t - \bu_{t-1} \|}{\eta_t} \\
    & \quad +\sum_{t=2}^T B_\psi(\bu_{t-1}, \bx_t ) \left( \frac{1}{\eta_t} - \frac{1}{\eta_{t-1}} \right) \\
    & \leq \frac{D^2}{\eta_1} + D^2 \left( \frac{1}{\eta_T} - \frac{1}{\eta_1} \right) + \gamma \sum_{t=2}^T \frac{\| \bu_t - \bu_{t-1} \|}{\eta_t} \\
    & = \frac{D^2}{\eta_T} + \gamma \sum_{t=2}^T \frac{\| \bu_t - \bu_{t-1} \|}{\eta_t}~.
\end{align*}
Adding $\ell_t(\bx_t)$ on both sides of \cref{eq:update_optimality} and summing over time yields the regret bound in \cref{eq:base_bound_dynamic}.
\end{proof}

Notice that the Lipschitz continuity assumption is not a strong requirement. Indeed, when the function $\psi$ is Lipschitz on $\mathcal{V}$, the Lipschitz condition on the Bregman divergence is automatically satisfied. When this is not true, we can still satisfy this condition changing the domain of interest. For example, in the case of learning with expert advice we have that $\gamma = \mathcal{O}(\ln T)$ if we use a ``clipped" simplex (details in \cref{sec:extension}).

We still have to set the learning rate in \cref{algo:dynamic_iomd} and next theorem shows how to do it in order to have a regret bound which is the minimum between temporal variability and path-length. The proof follows the one of \citet[Theorem 6.12]{campolongo2020temporal} and is reported in \cref{appendix:technical} for completeness.

\begin{restatable}{theorem}{adaImplicitBound}
\label{thm:adaImplicit_regret}
Let $\tau \ge 0$ be a positive constant. Under the assumptions of \cref{thm:dynamic_iomd_bound}, for any sequence $\bu_{1:T}$ whose path length $C_T(\bu_{1:T})$ is less or equal than $\tau$, \cref{algo:dynamic_iomd} with $1/\eta_t = \lambda_t = \frac{1}{\beta^2}\sum_{i=1}^{t-1} \delta_t $, and $\beta^2 = (D^2 + \gamma \tau)$ incurs dynamic regret against upper bounded as follows
\begin{align}
\label{eq:iomd_bound_1}
    R_T(\bu_{1:T}) &\leq \min \Big\{ 2(\ell_1(\bx_1) - \ell_T(\bx_{T+1}) + V_T), \nonumber\\
        & 2 \sqrt{(3D^2 + \gamma \tau) \textstyle\sum_{t=1}^T \| \bg_t \|_\star^2 } \Big\}~,
\end{align}
where $V_T = \sum_{t=2}^T \max_{\bx \in \mathcal{V}} \ell_t(\bx) - \ell_{t-1}(\bx)$. 

In addition, with $\beta^2 = D^2$ we get
\begin{align}
\label{eq:iomd_bound_2}
    R_T(\bu_{1:T}) &\le \left( 2 + \frac{\gamma C_T}{D^2} \right) \min \Big\{\ell_1(\bx_1) - \ell_T(\bx_{T+1}) + \nonumber\\
    & V_T,\, \sqrt{3D^2 \textstyle\sum_{t=1}^T \|\bg_t\|_\star^2}\Big\}~.
\end{align} 
\end{restatable}

If we assume an upper bound $\|\bg_t\|_\star^2 \le \max_{t \in [T]} \|\bg_t\|_\star^2 \le 1$, and that $\gamma = D = 1$, then the result in \cref{eq:iomd_bound_1} gives us a dynamic regret bound of $\mathcal{O}(\min\{V_T, \sqrt{T(1 + \tau)}\})$. This bound is tight for sequences whose path-length $C_T = \tau$, matching the lower bounds for both the path-length and temporal variability. Moreover, we show in \cref{appendix:technical} how the algorithm can be adapted using a doubling trick in the same spirit of \citet{jadbabaie2015online}, when $\tau$ is not fixed in advance but the path-length can be observed on the fly.

\paragraph{Running Time} Compared to standard \emph{Mirror Descent} the update depicted in \cref{eq:implicit_update} does not have a general solution, but it has to be calculated case-by-case depending on the loss function. There are some cases relevant in practical applications where this update is available in closed form, such as regression with square or absolute loss or classification with hinge loss (see for example \cite{crammer2006online}) and the time complexity of the update is similar to OMD. When a closed-form solution is not available, it can be approximated efficiently using numerical methods \citep{song2018fully,alex_blog}.

\subsection{Applications} 
\label{sec:extension}

Next, we are going to show some applications of this algorithm. We point out that the same set of applications and the related regret bounds continue to hold for the algorithm presented later in \cref{thm:combiner}, with some differences that will be highlighted in \cref{sec:combiner}.

\paragraph{Learning with Expert Advice} In this setting the loss is linear, i.e., $\ell_t(\bx) = \langle \bg_t, \bx \rangle$, therefore the implicit and the standard version of \emph{Mirror Descent} coincide. It is known that \emph{Mirror Descent} with negative entropy regularization, i.e., $\psi(\bx) = \sum_{i=1}^d x_i \ln x_i$, yields the exponential weights algorithm which is optimal. We recall that the Bregman divergence induced by the negative entropy is the KL divergence, for any 2 points on the simplex $\Delta_d$, which can be potentially unbounded. Indeed, using a dynamic learning rate with \emph{Mirror Descent} in this setting gives rise to a vacuous bound \citep[Theorem 4]{orabona2018scale}. For this reason, we modify the domain of interest: instead of the regular simplex, we use a ``clipped" version of it. 
\begin{equation} \label{eq:simplex_cut}
    \Delta_d^\alpha \triangleq \{ \bx \in \mathbb{R}^d_+: \|\bx\|_1=1, \, x_i \ge \frac{\alpha}{d}\, \forall i=1,\dots,d \}~.
\end{equation}
This set makes the diameter w.r.t to KL divergence bounded. Indeed, we have that
\begin{align}
\label{eq:max_KL}
    \max_{\bx, \by \in \Delta_d^\alpha} \textup{KL}(\bx, \by) \le \sum_{i=1}^d x_i \ln \frac{d}{\alpha} \le \ln \frac{d}{\alpha}~.
\end{align}
Furthermore, the update in \cref{eq:implicit_update} using this ``clipped" simplex can be computed efficiently, see \citet[Theorem 7]{herbster2001tracking} for details. 

We prove a regret bound for \cref{algo:dynamic_iomd} in the next theorem. We stress that the proof the next result does not follow directly from an application of \cref{thm:adaImplicit_regret} (it is reported in \cref{appendix:lea}).

\begin{restatable}{theorem}{learningExpertAdvice}
\label{thm:lea}
Consider the setting of Learning with Expert Advice on the $d$-dimensional simplex $\Delta_d$. Assume $0 \le g_{t,i} \le L_\infty$ for all $t=1,\dots,T$, and $i=1,\dots,d$. Assume that $T \ge d$ and set $\alpha = d / T$. Furthermore, set $\eta_{t+1}^{-1}=\lambda_{t+1} = \frac{1}{\beta^2} \sum_{i=1}^t \delta_t$, with $\delta_t$ defined as in \cref{thm:dynamic_iomd_bound}. Then, for any sequence of comparators $\bu_{1:T}$ with $\bu_t \in \Delta_d$ such that $C_T(\bu_{1:T}) \le \tau$, using $\beta^2 = (1+\tau)\ln T$ the regret of \cref{algo:dynamic_iomd} run on $\Delta_d^\alpha$ is bounded as
\begin{align}
\label{eq:lea_regret}
    R_T& (\bu_{1:T}) \le 2 \min \Big\{ \ell_1(\bx_1) - \ell_T(\bx_{T+1})+ V_T, \nonumber\\ 
        & \sqrt{(1 + (1+\tau)\ln T) \textstyle\sum_{t=1}^T \mathbb{E}_t[\bg_t^2]} \Big\} + 2 L_\infty d ~,
\end{align}
where $\mathbb{E}_t[\bg_t^2] = \sum_{t=1}^T x_{t,i} g_{t,i}^2$ and $V_T = \sum_{t=2}^T \max_{\bx \in \Delta_d} \ell_t(\bx) - \ell_{t-1}(\bx)$.
\end{restatable}

\paragraph{Discussion.}  Note that $\mathbb{E}_t[\bg_t^2] = \sum_{i=1}^d x_{t,i} g_{t,i}^2 \le \|\bg_t\|_\infty^2$. Furthermore, if the sequence $\bu_{1:T}$ is only composed by ``corners'' of the simplex, then $C_T(\bu_{1:T})$ is given roughly by the number of ``shifts'', i.e., $C_T(\bu_{1:T}) = 2\sum_{t=2}^T \mathbbm{1}\{\bu_t \neq \bu_{t-1} \}$. In this setting, we know that the \emph{Fixed-Share} algorithm \citep{herbster1998tracking} achieves a regret bound of $\mathcal{O}(\sqrt{T S \ln \frac{d T}{S}})$, where $C_T(\bu_{1:T}) = \mathcal{O}(S)$ is the number of shifts and is fixed in advance. Our algorithm achieves a similar bound, but it can be sometimes better. For example, assume $\ell_t(\bx) \in [0,1]$ and that $\ell_t$ stays fixed for all except $S$ rounds, then we have that $V_T = S$. In this case, \emph{Fixed-Share} regret is $\mathcal{O}(\sqrt{TS})$, while the bound in \cref{eq:lea_regret} reduces to $\mathcal{O}(S)$. It should be noted that both algorithms assume oracle knowledge of $S$. In order to remove this assumption, we refer to a different algorithm described in \cref{sec:combiner}. Moreover, it can be shown (see \cref{appendix:lea}) that the regret bound in \cref{eq:lea_regret} can lead to a first-order regret bound which depends on the loss of the sequence of competitors. To the best of our knowledge a similar result is not \emph{easily} obtainable for the \emph{Fixed-Share} algorithm (see \citet[Section 7.3]{cesa2012mirror} for a comparison).

\paragraph{Composite Losses} In this paragraph, we assume that the losses received are composed by two parts: one convex part changing over time and the other one fixed and known to the algorithm. These losses are called \emph{composite} \citep{duchi2010composite}. This setting was also studied in the implicit case in \citet{song2018fully} for the static regret. For example, we might have $\ell_t(\bx) = \tilde{\ell}_t(\bx) + \beta \| \bx \|_1$. In this case, considering a bounded domain $\mathcal{V}$ the update rule for \cref{algo:dynamic_iomd} will be
\begin{align} \label{update_composite}
    \bx_{t+1} & = \argmin_{\bx \in V} \ \tilde{\ell}_t(\bx) + \beta \| \bx \|_1  + B_\psi(\bx, \bx_t),
\end{align}
which will promote sparsity in our model.
We show in \cref{appendix:composite} that the algorithm \emph{AdaImplicit}~\citep{campolongo2020temporal} designed for the static regret already satisfies a regret bound order of $\mathcal{O}(\min( V_T, \sqrt{T} ))$ (improving over the existing result of \citet{song2018fully}).
Next, we extend the analysis to the dynamic scenario. In particular, using \cref{algo:dynamic_iomd} we can give the following theorem, whose proof is reported in \cref{appendix:composite}.

\begin{restatable}{theorem}{dynamicComposite} 
\label{regret:composite_dynamic}
Let $\mathcal{V} \subset X \subseteq \mathbb{R}^d $ be a non-empty closed convex set. Let $\ell_t(\bx) = \tilde{\ell}_t(\bx) + r(\bx) $, where $r: X \rightarrow \mathbb{R}$ is a convex function. Then, under the assumptions of \cref{thm:dynamic_iomd_bound} the regret of \cref{algo:dynamic_iomd} run with $1/\eta_t = \lambda_t = \frac{1}{\beta^2} \sum_{i=1}^{t-1} \delta_i $ and $\beta^2 = D^2 + \gamma \tau$ against any sequence of comparators $\bu_{1:T}$ whose path-length $C_T$ is less or equal than $\tau$ is bounded as
\begin{align*}
    R_T(\bu_{1:T}) \leq \min \Big\{ &2 (\ell_1(\bx_1) - \ell_{T+1}(\bx_{T+1}) + V_T),\, \\ 
        & 2 \sqrt{(3D^2 + \gamma \tau) \textstyle\sum_{t=1}^T \| \bg_t \|_\star^2} \Big\},
\end{align*}
where $V_T = \sum_{t=2}^T \max_{\bx \in \mathcal{V}} \tilde\ell_t(\bx) - \tilde\ell_{t-1}(\bx)$ and $\bg_t \in \partial \ell_t(\bx_t)$.
\end{restatable}

\paragraph{Remark.} Comparing to the regret bound given in \cref{thm:adaImplicit_regret}, we can see that  the result above contains a subtle difference: the temporal variability is given only in terms of the variable part of the losses, $\tilde\ell_t$.

All the results up to this point are given under the assumption that the class of strategies we want to compete against is fixed before the start of the game, i.e., an upper bound to $C_T$ is fixed beforehand. This can be limiting in practice: in a truly realistic online setting, knowing the right upper bound beforehand might be hard. Therefore, in the next section we are going to provide an algorithm which adapts to the values of $C_T$ for any possible sequence of comparators, but at the same time guarantees a bound in $V_T$.

\section{Adapting to different path-lengths} 
\label{sec:combiner}

\begin{algorithm}[t]
\caption{Anytime $(\mathcal{A,B})$-PROD}
\label{algo:prod}
\begin{algorithmic}[1]
{
\REQUIRE{Algorithms $\mathcal{A,B}$, $\eta_1= w_{1,\mathcal{A}}=w_{1,\mathcal{B}}=1/2$}
\FOR{$t=1,\dots,T$}
\STATE{Let $p_{t,\mathcal{A}} = \frac{\eta_t w_{t,\mathcal{A}}}{\eta_t w_{t,\mathcal{A}}+ w_{t,\mathcal{B}}/2}$, $p_{t,\mathcal{B}} = 1-p_{t,\mathcal{A}}$}
\STATE{Get $\ba_t$ from $\mathcal{A}$ and $\bb_t$ from $\mathcal{B}$}
\STATE{Set $\bx_t = p_{t,\mathcal{A}} \ba_t + p_{t,\mathcal{B}} \bb_t$}
\STATE{Receive $\ell_t: \mathbb{R}^d \rightarrow [0,1]$ and pay $\ell_t(\bx_t)$}
\STATE{Feed $\ell_t$ to $\mathcal{A}$ and $\mathcal{B}$}
\STATE{Set $r_{t} = \ell_t(\bb_t) - \ell_t(\ba_t) $}
\STATE{Set $\eta_{t+1}=\sqrt{(1 + \sum_{i=1}^t (\ell_i(\bb_i) - \ell_i(\ba_i))^2)^{-1}}$}
\STATE{Set $w_{t+1,\mathcal{A}} = w_{t,\mathcal{A}}(1 + \eta_t r_{t})^{\eta_{t+1}/\eta_t}  $}
\ENDFOR
}
\end{algorithmic}
\end{algorithm}

In this section, we present an approach to obtain the optimal bound of $\mathcal{O}(\min\{V_T, \sqrt{T(1+C_T)}\})$ on the dynamic regret for all the sequences of comparators simultaneously. Our approach is based on smartly combining different algorithms. 

Using existing algorithms \citep[see, e.g.,][]{cutkosky2020parameter,zhang2018adaptive} we can achieve the optimal dynamic regret bound for all possible sequences of comparators.
In particular, a recent result from \citet{cutkosky2020parameter} shows the condition that strongly-adaptive algorithms need to satisfy in order to incur the optimal bound of $\mathcal{O}(\sqrt{T(1+C_T)})$, which we report in the next theorem (proof in \cref{appendix:combiner}).

\begin{restatable}[Adapted from \citet{cutkosky2020parameter}]{theorem}{combinerDynamic}
\label{thm:strongly_adaptive_path_length}
Let $\mathcal{V} \subset X \subset \mathbb{R}^d $ be non-empty closed convex sets, $\psi: X \rightarrow \mathbb{R}$. Define $D^2 = \max_{\bx, \by \in \mathcal{V}} B_\psi(\bx, \by)$. Given a sequence of loss functions $\ell_1,\dots,\ell_T$, assume there exists an algorithm that for any interval $I = [s, e] \subseteq [1, T]$ and sequence $\bu_{s:e}$ guarantees a dynamic regret bounded as
\begin{equation}
\label{eq:cutkosky_condition_1}
R_I(\bu_{s:e}) = \mathcal{O}\left((D + C_I)\sqrt{|I|}\right)~,
\end{equation}
where $|I|=(e-s)$.
Then, for any interval $J=[s',e'] \subseteq [1, T]$ it also guarantees
\begin{equation}
\label{eq:cutkosky_dynamic}
    R_I(\bu_{s':e'}) \leq \mathcal{O} \left( \sqrt{|J| D(C_J + D)} \right)~,
\end{equation}
where $C_J = \sum_{t=s'+1}^{e'} \| \bu_t - \bu_{t-1} \|$.
\end{restatable}

On the other hand, in order to achieve a bound on the temporal variability, we can simply adopt the greedy strategy given in the \cref{algo:greedy}, which in every step plays the minimizer of the last seen loss function. Hence, we need to combine these two strategies.

A plain expert algorithm would fail to give a constant bound w.r.t. the temporal variability. Instead, we use a modification of the \emph{ML-Prod} algorithm \citep{gaillard2014second} proposed in \citet{sani2014exploiting} and depicted in \Cref{algo:prod}. This algorithm takes as input two base learners (aka experts) $\mathcal{A}$ and $\mathcal{B}$ and guarantees a regret which is (almost) constant against $\mathcal{B}$ and $\mathcal{O}(\sqrt{T} \ln\ln T)$ against $\mathcal{A}$ in the worst case. The idea is to use a strongly adaptive algorithm which satisfies the assumptions of \cref{thm:strongly_adaptive_path_length} such as the one from \citet{cutkosky2020parameter} as algorithm $\mathcal{A}$, and the greedy strategy in \cref{algo:greedy} as $\mathcal{B}$.
In the next theorem, we provide an upper bound to the dynamic regret of the resulting algorithm (proof in \cref{appendix:combiner}).

\begin{restatable}{theorem}{combinerTheorem}
\label{thm:combiner}
Let $\mathcal{V} \subset X \subset \mathbb{R}^d $ be non-empty closed convex sets, $\psi: X \rightarrow \mathbb{R}$. Define $D^2 = \max_{\bx, \by \in \mathcal{V}} B_\psi(\bx, \by)$. Let $\ell_1, \dots, \ell_t$ be a sequence of convex loss functions such that $\ell_t : \mathbb{R}^d \rightarrow [0, 1]$ for all $t$. Then, for all sequence of $\bu_{1:T}$ with path-length $C_T$, running \Cref{algo:prod} with $\mathcal{A}$ as a strongly-adaptive algorithm satisfying the condition from \cref{thm:strongly_adaptive_path_length}, and $\mathcal{B}$ as \Cref{algo:greedy}, guarantees
\[
  R_T (\bu_{1:T})
  \le \mathcal{O}\left(\min\left\{V_T, \sqrt{C}+ \sqrt{TD(C_T + D)} \right\} \right) ,
\]
where $C$ is an upper bound to the loss of algorithm $\mathcal{A}$.
\end{restatable}

\paragraph{Discussion.} The algorithm from last theorem is able to guarantee the optimal dependence on $C_T$ for any sequence of comparators $\bu_{1:T}$, without requiring any prior knowledge. However, compared to the algorithms from \cref{sec:dynamic_iomd}, \cref{thm:combiner} requires the losses to be bounded in a range known to the algorithm. Furthermore, note that all known strongly-adaptive algorithms require a running time of $\mathcal{O}(d \ln T)$ per-update, which is higher than $\mathcal{O}(d)$ required by OMD (and its Implicit version when the update is available in closed form).
Finally, we point out that the actual regret bound depends on the specific choice of the algorithm $\mathcal{A}$ used. Next, we sketch some applications in the same spirit of \cref{sec:extension}. 

\paragraph{Applications.} To make a comparison with the setting of Learning with Expert Advice previously covered, consider again the ``shifting'' scenario with $S$ shifts of \cref{sec:extension}. In this case we can adopt a strategy similar to the strongly-adaptive algorithm \emph{CBCE} from \citet{jun2017improved}. It is known that a strongly-adaptive algorithm incurs a dynamic regret bounded by $\tilde{\mathcal{O}}(\sqrt{TS})$, without requiring the knowledge of $S$ in advance \citep[see, e.g., Appendix A in][]{jun2017improved}, contrarily to \cref{algo:dynamic_iomd}. Moreover, from our bound in \cref{thm:combiner} if the loss functions stay fixed for all but $S$ rounds, thanks to the guarantee of \cref{algo:greedy} the regret bound is $\mathcal{O}(V_T) = \mathcal{O}(S)$, which again is like the guarantee of \cref{algo:dynamic_iomd} with the important difference that knowing in advance the number of shifts is not required.

On the other hand, when considering for example the setting of composite losses, or more in general Euclidean domains and the $L_2$ norm, we can adopt the algorithm from \citet{cutkosky2020parameter}, which is adaptive to the sum of the gradients $\sum_{t=1}^T \| \bg_t \|_2^2 $ and therefore can have a potentially better bound compared to \emph{CBCE}, which has a worst-case regret bound of $\mathcal{O}(\sqrt{T})$.

\section{Conclusion}

In this work, we have shown that existing bounds in the dynamic setting with full information feedback can be improved, by establishing a lower bounds on the dynamic regret in terms of temporal variability of the loss functions and showing algorithms with matching upper bounds. In particular, we designed an algorithm using implicit updates that can adapt to both the temporal variability and the path-length of the sequence of comparators. Furthermore, when the desired path-length is not fixed in advance, we showed how to combine existing algorithms in order to achieve the optimal bound.

An interesting question remains open: is it possible to obtain a dynamic regret bound of $\mathcal{O}(\min\{V_T, \sqrt{T(1+C_T)}\})$ for all sequence of comparators with a single algorithm? If so, what is its running time? As observed in previous work, all strongly-adaptive algorithms \citep{cutkosky2020parameter,jun2017improved} have a running time of $\mathcal{O}(T \ln T)$ and it is currently not known whether it can be improved. Future research directions therefore could aim at designing faster and more practical algorithms which can adapt to unknown path-lengths, or in alternative prove that this goal cannot be achieved.

\section*{Acknowledgements}
This material is based upon work supported by the National Science Foundation under grants no. 1925930 ``Collaborative Research: TRIPODS Institute for Optimization and Learning'' and no. 1908111 ``AF: Small: Collaborative Research: New Representations for Learning Algorithms and Secure Computation''.

\balance
\bibliography{sample}

\begin{thebibliography}{31}
\providecommand{\natexlab}[1]{#1}
\providecommand{\url}[1]{\texttt{#1}}
\expandafter\ifx\csname urlstyle\endcsname\relax
  \providecommand{\doi}[1]{doi: #1}\else
  \providecommand{\doi}{doi: \begingroup \urlstyle{rm}\Url}\fi

\bibitem[Ajalloeian et~al.(2020)Ajalloeian, Simonetto, and
  Dall’Anese]{ajalloeian2020inexact}
Amirhossein Ajalloeian, Andrea Simonetto, and Emiliano Dall’Anese.
\newblock Inexact online proximal-gradient method for time-varying convex
  optimization.
\newblock In \emph{2020 American Control Conference (ACC)}, pages 2850--2857.
  IEEE, 2020.

\bibitem[Besbes et~al.(2015)Besbes, Gur, and Zeevi]{besbes2015non}
Omar Besbes, Yonatan Gur, and Assaf Zeevi.
\newblock Non-stationary stochastic optimization.
\newblock \emph{Operations research}, 63\penalty0 (5):\penalty0 1227--1244,
  2015.

\bibitem[Campolongo and Orabona(2020)]{campolongo2020temporal}
Nicol{\`o} Campolongo and Francesco Orabona.
\newblock Temporal variability in implicit online learning.
\newblock \emph{arXiv preprint arXiv:2006.07503}, 2020.

\bibitem[Cesa-Bianchi et~al.(2012)Cesa-Bianchi, Gaillard, Lugosi, and
  Stoltz]{cesa2012mirror}
N~Cesa-Bianchi, P~Gaillard, G~Lugosi, and G~Stoltz.
\newblock Mirror descent meets fixed share (and feels no regret).
\newblock In \emph{Conference on Neural Information Processing Systems},
  volume~2, pages 989--997. Neural information processing systems foundation,
  2012.

\bibitem[Cesa-Bianchi et~al.(2004)Cesa-Bianchi, Conconi, and
  Gentile]{cesa2004generalization}
Nicolo Cesa-Bianchi, Alex Conconi, and Claudio Gentile.
\newblock On the generalization ability of on-line learning algorithms.
\newblock \emph{IEEE Transactions on Information Theory}, 50\penalty0
  (9):\penalty0 2050--2057, 2004.

\bibitem[Crammer et~al.(2006)Crammer, Dekel, Keshet, Shalev-Shwartz, and
  Singer]{crammer2006online}
Koby Crammer, Ofer Dekel, Joseph Keshet, Shai Shalev-Shwartz, and Yoram Singer.
\newblock Online passive-aggressive algorithms.
\newblock \emph{Journal of Machine Learning Research}, 7\penalty0
  (Mar):\penalty0 551--585, 2006.

\bibitem[Cutkosky(2020)]{cutkosky2020parameter}
Ashok Cutkosky.
\newblock Parameter-free, dynamic, and strongly-adaptive online learning.
\newblock In \emph{International Conference on Machine Learning}, volume~2,
  2020.

\bibitem[Daniely et~al.(2015)Daniely, Gonen, and
  Shalev-Shwartz]{daniely2015strongly}
Amit Daniely, Alon Gonen, and Shai Shalev-Shwartz.
\newblock Strongly adaptive online learning.
\newblock In \emph{International Conference on Machine Learning}, pages
  1405--1411, 2015.

\bibitem[Dixit et~al.(2019)Dixit, Bedi, Tripathi, and Rajawat]{dixit2019online}
Rishabh Dixit, Amrit~Singh Bedi, Ruchi Tripathi, and Ketan Rajawat.
\newblock Online learning with inexact proximal online gradient descent
  algorithms.
\newblock \emph{IEEE Transactions on Signal Processing}, 67\penalty0
  (5):\penalty0 1338--1352, 2019.

\bibitem[Duchi et~al.(2010)Duchi, Shalev-Shwartz, Singer, and
  Tewari]{duchi2010composite}
John~C Duchi, Shai Shalev-Shwartz, Yoram Singer, and Ambuj Tewari.
\newblock Composite objective mirror descent.
\newblock In \emph{COLT}, pages 14--26. Citeseer, 2010.

\bibitem[Gaillard et~al.(2014)Gaillard, Stoltz, and
  Van~Erven]{gaillard2014second}
Pierre Gaillard, Gilles Stoltz, and Tim Van~Erven.
\newblock A second-order bound with excess losses.
\newblock In \emph{Conference on Learning Theory}, pages 176--196, 2014.

\bibitem[Hall and Willett(2013)]{hall2013dynamical}
Eric Hall and Rebecca Willett.
\newblock Dynamical models and tracking regret in online convex programming.
\newblock In \emph{International Conference on Machine Learning}, pages
  579--587, 2013.

\bibitem[Hazan and Seshadhri(2007)]{hazan2007adaptive}
Elad Hazan and Comandur Seshadhri.
\newblock Adaptive algorithms for online decision problems.
\newblock In \emph{Electronic colloquium on computational complexity (ECCC)},
  volume~14, 2007.

\bibitem[Herbster and Warmuth(1998)]{herbster1998tracking}
Mark Herbster and Manfred~K Warmuth.
\newblock Tracking the best expert.
\newblock \emph{Machine learning}, 32\penalty0 (2):\penalty0 151--178, 1998.

\bibitem[Herbster and Warmuth(2001)]{herbster2001tracking}
Mark Herbster and Manfred~K Warmuth.
\newblock Tracking the best linear predictor.
\newblock \emph{Journal of Machine Learning Research}, 1\penalty0
  (281-309):\penalty0 10--1162, 2001.

\bibitem[Jadbabaie et~al.(2015)Jadbabaie, Rakhlin, Shahrampour, and
  Sridharan]{jadbabaie2015online}
Ali Jadbabaie, Alexander Rakhlin, Shahin Shahrampour, and Karthik Sridharan.
\newblock Online optimization: Competing with dynamic comparators.
\newblock In \emph{Artificial Intelligence and Statistics}, pages 398--406,
  2015.

\bibitem[Jun et~al.(2017)Jun, Orabona, Wright, and Willett]{jun2017improved}
Kwang-Sung Jun, Francesco Orabona, Stephen Wright, and Rebecca Willett.
\newblock Improved strongly adaptive online learning using coin betting.
\newblock In \emph{Artificial Intelligence and Statistics}, pages 943--951.
  PMLR, 2017.

\bibitem[Kivinen and Warmuth(1997)]{kivinen1997exponentiated}
Jyrki Kivinen and Manfred~K Warmuth.
\newblock Exponentiated gradient versus gradient descent for linear predictors.
\newblock \emph{information and computation}, 132\penalty0 (1):\penalty0 1--63,
  1997.

\bibitem[Moreau(1965)]{moreau1965proximite}
Jean-Jacques Moreau.
\newblock Proximit{\'e} et dualit{\'e} dans un espace hilbertien.
\newblock \emph{Bulletin de la Soci{\'e}t{\'e} math{\'e}matique de France},
  93:\penalty0 273--299, 1965.

\bibitem[Nesterov(2013)]{nesterov2013introductory}
Yurii Nesterov.
\newblock \emph{Introductory lectures on convex optimization: A basic course},
  volume~87.
\newblock Springer Science \& Business Media, 2013.

\bibitem[Orabona(2019)]{Orabona19}
F.~Orabona.
\newblock A modern introduction to online learning.
\newblock \emph{arXiv preprint arXiv:1912.13213}, 2019.

\bibitem[Orabona and P{\'{a}}l(2018)]{orabona2018scale}
Francesco Orabona and D{\'{a}}vid P{\'{a}}l.
\newblock Scale-free online learning.
\newblock \emph{Theor. Comput. Sci.}, 716:\penalty0 50--69, 2018.
\newblock \doi{10.1016/j.tcs.2017.11.021}.
\newblock URL \url{https://doi.org/10.1016/j.tcs.2017.11.021}.

\bibitem[Parikh and Boyd(2014)]{parikh2014proximal}
Neal Parikh and Stephen Boyd.
\newblock Proximal algorithms.
\newblock \emph{Foundations and Trends in optimization}, 1\penalty0
  (3):\penalty0 127--239, 2014.

\bibitem[Sani et~al.(2014)Sani, Neu, and Lazaric]{sani2014exploiting}
Amir Sani, Gergely Neu, and Alessandro Lazaric.
\newblock Exploiting easy data in online optimization.
\newblock In \emph{Advances in Neural Information Processing Systems}, pages
  810--818, 2014.

\bibitem[Shalev-Shwartz(2012)]{shalev2012online}
Shai Shalev-Shwartz.
\newblock Online learning and online convex optimization.
\newblock \emph{Foundations and Trends{\textregistered} in Machine Learning},
  4\penalty0 (2):\penalty0 107--194, 2012.

\bibitem[Shtof(2020)]{alex_blog}
Alex Shtof.
\newblock Proximal point - convex on linear losses, 2020.
\newblock URL
  \url{https://alexshtf.github.io/2020/02/15/ProximalConvexOnLinear.html}.

\bibitem[Song et~al.(2018)Song, Liu, Liu, Jiang, and Zhang]{song2018fully}
Chaobing Song, Ji~Liu, Han Liu, Yong Jiang, and Tong Zhang.
\newblock Fully implicit online learning.
\newblock \emph{arXiv preprint arXiv:1809.09350}, 2018.

\bibitem[Yang et~al.(2016)Yang, Zhang, Jin, and Yi]{yang2016tracking}
Tianbao Yang, Lijun Zhang, Rong Jin, and Jinfeng Yi.
\newblock Tracking slowly moving clairvoyant: Optimal dynamic regret of online
  learning with true and noisy gradient.
\newblock In \emph{International Conference on Machine Learning}, pages
  449--457, 2016.

\bibitem[Zhang et~al.(2018{\natexlab{a}})Zhang, Lu, and
  Zhou]{zhang2018adaptive}
Lijun Zhang, Shiyin Lu, and Zhi-Hua Zhou.
\newblock Adaptive online learning in dynamic environments.
\newblock In \emph{Advances in neural information processing systems}, pages
  1323--1333, 2018{\natexlab{a}}.

\bibitem[Zhang et~al.(2018{\natexlab{b}})Zhang, Yang, Jin, and
  Zhou]{zhang2018dynamic}
Lijun Zhang, Tianbao Yang, Rong Jin, and Zhi-Hua Zhou.
\newblock Dynamic regret of strongly adaptive methods.
\newblock In \emph{International Conference on Machine Learning}, pages
  5882--5891, 2018{\natexlab{b}}.

\bibitem[Zinkevich(2003)]{zinkevich2003online}
Martin Zinkevich.
\newblock Online convex programming and generalized infinitesimal gradient
  ascent.
\newblock In \emph{Proceedings of the 20th international conference on machine
  learning (icml-03)}, pages 928--936, 2003.

\end{thebibliography}

\newpage
\onecolumn
\appendix

\section{Dynamic IOMD} 
\label{appendix:technical}

In this section we prove a regret bound for \cref{algo:dynamic_iomd} and its variant using a doubling trick mentioned in \cref{sec:dynamic_iomd}. 

The following lemma is a generalization of {\citet[Lemma 7.12]{Orabona19}} which we will use in subsequent proofs.

\begin{lemma} \label{lemma:adahedge}
Let ${\{a_t\}_{t=1}^\infty}, \{b_t\}_{t=1}^\infty$ be two sequence of non-negative real numbers. Suppose that ${\{\Delta_t\}_{t=1}^\infty}$ is a sequence of non-negative real numbers satisfying $\Delta_1 = 0$ and\footnote{With a small abuse of notation, let $\min(x, y/0) = x$.} $\Delta_{t+1} \le \Delta_t + \min \left\{d b_t, \ c a_t^2/(2\Delta_t) \right\}$, for any $t \ge 1$. 
Then, for any ${T \ge 0}, {\Delta_{T+1} \le \sqrt{d^2 \sum_{t=1}^T b_t^2 + c\sum_{t=1}^T a_t^2}}$.
\end{lemma}
\begin{proof}
Observe that
\begin{align*}
\Delta_{T+1}^2 
 &= \sum_{t=1}^T \Delta_{t+1}^2 - \Delta_{t}^2 = \sum_{t=1}^T \underbrace{(\Delta_{t+1} - \Delta_t)^2}_{\textup{(a)}} + \sum_{t=1}^T \underbrace{2 (\Delta_{t+1} - \Delta_t) \Delta_t}_{\textup{(b)}}~.
\end{align*}
We bound the sequences (a) and (b) separately. For (a), from the assumption on the recurrence and using the first term in the minimum we have that $(\Delta_{t+1} - \Delta_t)^2 \leq d^2 b_t^2 $. On the other hand, for (b) using the second term in the minimum in the recurrence we get $2(\Delta_{t+1} - \Delta_t) \Delta_t \leq c a_t^2 $. Putting together the results we have that $\Delta_{T+1}^2 \leq d^2 \sum_{t=1}^T b_t^2 + c \sum_{t=1}^T a_t^2$ and the lemma follows.
\end{proof}

We are now ready to show a regret bound for \cref{algo:dynamic_iomd}, as stated in \cref{thm:adaImplicit_regret}. The statement of the theorem is reported next for completeness.
\adaImplicitBound*
\begin{proof}
First, note that $(\lambda_t)_{t=1}^T$ is an increasing sequence, since $\delta_t \geq 0$. Indeed, from the optimality of the update rule of \cref{algo:dynamic_iomd} we have
\begin{align*}
    \ell_t(\bx_{t+1}) + \lambda_t B_\psi(\bx_{t+1}, \bx_t) &\leq \ell_t(\bx_t) + \lambda_t B_\psi(\bx_t, \bx_t) = \ell_t(\bx_t)~,
\end{align*}
which implies $\delta_t := \ell_t(\bx_t) - \ell_t(\bx_{t+1}) - \lambda_t B_\psi(\bx_{t+1}, \bx_t) \ge 0$.
Hence, by using the prescribed learning rate $\lambda_t$, we can rewrite the bound in \cref{eq:base_bound_dynamic} as follows
\begin{align*}
    R_T(\bu_{1:T}) & \leq \lambda_T D^2 + \gamma \sum_{t=2}^T \lambda_t \| \bu_t - \bu_{t-1} \| + \beta^2 \lambda_{T+1} \nonumber \\
    & \leq (D^2 + \beta^2) \lambda_{T+1} + \gamma \lambda_{T+1} \sum_{t=2}^T \| \bu_t - \bu_{t-1} \| \nonumber \\
    & \le ( D^2 + \beta^2 + \gamma \tau) \lambda_{T+1},
\end{align*}
where in the second inequality we have used the fact that $(\lambda_t)_{t=1}^T$ is an increasing sequence.

The rest of the proof is similar to the one in \citet[Theorem 6.2]{campolongo2020temporal}. From the choice of $\lambda_t$, we have that
\begin{align} \label{eq:helper_1}
    \beta^2 \lambda_{T+1} & = \sum_{t=1}^T \delta_t = \sum_{t=1}^T \ell_t(\bx_t) - \ell_t(\bx_{t+1}) - \lambda_t B_\psi(\bx_{t+1}, \bx_t) \nonumber \\ &\le \sum_{t=1}^T \ell_t(\bx_t) - \ell_t(\bx_{t+1}) \nonumber \\
    &\le \ell_1(\bx_1) - \ell_T(\bx_{T+1}) + \sum_{t=2}^T \max_{\bx \in \mathcal{V}} \ell_t(\bx) - \ell_{t-1}(\bx) \nonumber \\
    & = \ell_1(\bx_1) - \ell_T(\bx_{T+1}) + V_T~,
\end{align}
where the first inequality derives from the fact that Bregman divergences are always positive. 

On the other hand, from the definition of $\delta_t$ we have that
\begin{equation}
\label{eq:adaimplicit_intermediate}
    \delta_t \le \ell_t(\bx_t) - \ell_t(\bx_{t+1}) \le \langle \bg_t, \bx_t - \bx_{t+1} \rangle \le \|\bg_t\|_\star \|\bx_t-\bx_{t+1}\|~.
\end{equation}

Now, note that from the assumptions in \cref{thm:dynamic_iomd_bound} we have that for any $\bx, \by \in \mathcal{V}$
\begin{equation*}
    D^2 \ge B_\psi(\bx, \by) \ge \frac12 \|\bx-\by\|^2~.
\end{equation*}
Therefore, $\|\bx-\by\| \le \sqrt2 D$ and substituting back in \cref{eq:adaimplicit_intermediate}
\begin{equation*}
    \delta_t \le \sqrt2 D \|\bg_t\|_\star.
\end{equation*}
On the other hand, by not discarding the negative Bregman divergence term in \cref{eq:adaimplicit_intermediate} we get
\begin{align*}
    \delta_t \le \|\bg_t\|_\star \sqrt{2 B_\psi(\bx_{t+1}, \bx_t)} - \lambda_t B_\psi(\bx_{t+1}, \bx_t) \le \frac{\|\bg_t\|_\star^2}{2 \lambda_t}~,
\end{align*}
where the last step derives from the fact that $bx - \frac{a}{2}x^2 \le \frac{b^2}{2a}, \forall x\in\mathbb{R}$, with $x = \sqrt{B_\psi(\bx_{t+1}, \bx_t)}$. 

To summarize, we have that 
\begin{equation*}
    \lambda_{t+1} = \lambda_t + \delta_t \le \lambda_t + \frac{1}{\beta^2} \min\left\{ \sqrt2 D\|\bg_t\|_\star, \frac{\|\bg_t\|_\star^2}{2\lambda_t} \right\}~.
\end{equation*}
Now, applying \cref{lemma:adahedge} with $\Delta_t = \lambda_t, a_t = b_t = \|\bg_t\|_\star,\, d=\frac{\sqrt2 D}{\beta^2},\, c=\frac{1}{\beta^2}$ yields
\begin{equation} 
\label{eq:helper_2}
    \lambda_{T+1} \leq \sqrt{\left(\frac{2D^2}{\beta^4} + \frac{1}{\beta^2}\right) \sum_{t=1}^T \| \bg_t \|_\star^2 }~.
\end{equation}
Therefore, putting together \cref{eq:helper_1,eq:helper_2} and using the suggested values for $\beta$, we get the stated results.
\end{proof}

\subsection{Adapting on the fly}
\label{sec:doubling_trick}

\begin{algorithm}[t]
\caption{Dynamic IOMD with Doubling Trick}
\label{algo:dynamic_adaImplicit}
\begin{algorithmic}[1]
{
\REQUIRE{Non-empty closed convex set $\mathcal{V} \subset X \subset \mathbb{R}^d$, $\psi: X \rightarrow \mathbb{R}$, $\bx_1 \in \mathcal{V}$, $\gamma$ such that $B_\psi(\bx, \bz) - B_\psi(\by, \bz) \leq \gamma \| \bx - \by \|,\, \forall \bx, \by, \bz \in \mathcal{V} $, $\beta^2_0 > 0$, observable sequence $\bu_{1:T}$}
\STATE{$i\leftarrow0$, $\lambda_1^0 \leftarrow 0$, $Q_0 \leftarrow \sqrt2 D$, $C_0 \leftarrow 0$}
\FOR{$t=1,\dots,T$}
\STATE{Output $\bx_t \in V$}
\STATE{Receive $\ell_t: \mathbb{R}^d \rightarrow \R$ and pay $\ell_t(\bx_t)$}
\STATE{Update $C_i \leftarrow C_i + \| \bu_t - \bu_{t-1} \| $}
\IF{$C_i > Q_i $}
    \STATE{$i \leftarrow i + 1$}
    \STATE{$Q_i \leftarrow \sqrt2 D 2^i, \, \lambda_{t+1}^i \leftarrow 0,\, C_i \leftarrow 0 $, $\beta^2_i \leftarrow D^2 + \gamma Q_i $}
    \STATE{Update $\bx_{t+1} \leftarrow \bx_t$}
\ELSE
    \STATE{Update $\bx_{t+1} \leftarrow \arg\min_{\bx \in \mathcal{V}}\ \ell_t(\bx) + \lambda_{t}^i B_{\psi}(\bx, \bx_t)  $}
    \STATE{Set $\delta_{t} \leftarrow \ell_t(\bx_t) - \ell_t(\bx_{t+1}) - \lambda_{t}^i B_\psi(\bx_{t+1}, \bx_{t})$}
    \STATE{Update $\lambda_{t+1}^i \leftarrow \lambda_{t}^i + \frac{1}{\beta^2_i} \delta_t $}
\ENDIF
\ENDFOR
}
\end{algorithmic}
\end{algorithm}

The result given in the previous paragraph was limited to all sequences of comparators whose path-length is fixed beforehand. Following \citet{jadbabaie2015online}, the approach given above can be generalized to any sequence of $\bu_{1:T}$ whose path length $C_T$ can be calculated on the fly. 

\paragraph{Doubling trick.} The idea is to run \Cref{algo:dynamic_iomd} in phases and tune the learning rate $\lambda_t$ appropriately. At the beginning of each phase $i$, we start monitoring the path length $C_i$. Once it reaches a certain threshold, we restart the algorithm doubling the threshold. Formally, we introduce a quantity $ Q_i $ for phase $i$ and set the learning rate $\lambda_t$ of the algorithm as $\lambda_t^i = \frac{1}{\beta^2_i}\sum_{s=1}^{t-1} \delta_s$, with $\beta^2_i = D^2 + \gamma Q_i $. The resulting algorithm is shown in \Cref{algo:dynamic_adaImplicit}.

We are now going to analyze the regret bound incurred by \Cref{algo:dynamic_adaImplicit}. First, we need the following lemma which bounds the number of times the algorithm is restarted.
\begin{lemma}
\label{lemma:max_epoch}
Let $t_i$ be the first time-step of epoch $i$, with $t_0 = 1$. Suppose \Cref{algo:dynamic_adaImplicit} is run for a total of $N + 1$ epochs. Let $C_i = \sum_{t=t_i}^{t_{i+1}-1} \| \bu_t - \bu_{t-1} \|$, with $ \| \bu_1- \bu_0 \| \triangleq 0 $. Let $C_T = \sum_{i=0}^N C_i$. Then, we have that $N$ satisfies 
\begin{equation}
\label{eq:epochs_bound}
    N \leq \log_2 \left( \frac{C_T}{\sqrt2 D} + 1 \right)~.
\end{equation}
\end{lemma}
\begin{proof}
First, recall that $\sum_{i=0}^{N-1} a^i = \frac{a^N - 1}{a - 1} $. Now, note that the sum in the first $N$ epochs of the quantity we are monitoring is at most equal to the final sum over all $N + 1$ epochs.

Therefore, we have the following
\begin{align*}
    \sum_{i=0}^{N-1} \sqrt2 D 2^i & \leq \sqrt2 D (2^N - 1) \leq \sum_{i=0}^N \sum_{t=t_i}^{t_{i+1}-1} \| \bu_t - \bu_{t-1} \|  = C_T,
\end{align*}
where $ \| \bu_{t_0} - \bu_{t_{0} - 1} \| = \| \bu_1 - \bu_0 \| \triangleq 0$ by definition.
Solving for $N$ yields the desired result.
\end{proof}

Next, we provide a theorem which gives a regret bound to \Cref{algo:dynamic_adaImplicit}.

\begin{theorem} 
\label{thm:adaimplicit_doubling_trick}
Let $\mathcal{V} \subset X \subset \mathbb{R}^d$ be a non-empty closed convex set. Assume \Cref{algo:dynamic_adaImplicit} is run for $N$ epochs. Then, under the assumptions of \Cref{thm:dynamic_iomd_bound} the regret against any sequence of comparators $\bu_{1:T}$ with $\bu_t \in \mathcal{V}$ is bounded as
\begin{align}
\label{eq:diomd_doubling_trick}
    R_T(\bu_{1:T}) &\leq (2 + c) \min \bigg( (\ell_1(\bx_1) - \ell_T( \bx_{T+1} ) + V_T),\, \sqrt{ \left(3 D^2 \left(\log_2 \tfrac{C_T}{\sqrt2 D} + 1 \right) + \gamma C_T \right) \textstyle\sum_{t=1}^T \| \bg_t \|_\star^2 } \bigg),
\end{align}
where $ c \triangleq \frac{\sqrt2 }{D + \gamma \sqrt2} $ and $C_T = \sum_{t=2}^{T} \| \bu_t - \bu_{t-1} \|$.
\end{theorem}
\begin{proof}
Let $V_i = \sum_{t=t_i + 1}^{t_{i+1}-1} \max_{\bx \in V} | \ell_t(\bx) - \ell_{t-1}(\bx) | $. Using the result from \Cref{thm:adaImplicit_regret}, assuming the knowledge of $C_i$ during each phase $i$ we have that
\begin{align*}
    R(\bu_{1:T}) = \sum_{i=0}^N \sum_{t=t_i}^{t_{i+1}-1} R(\bu_{t_i:t_{i+1}-1})
    \leq \sum_{i=0}^N \sum_{t=t_i}^{t_{i+1}-1} \frac{D^2 + \gamma C_i + \beta_i^2 }{\beta_i^2} \min( B_1^i, B_2^i ),
\end{align*}
where in the last inequality we used \cref{thm:adaImplicit_regret} with $B_1^i = \ell_{t_i}(\bx_1) - \ell_{t_{i+1}-1}(\bx_{t_{i+1}}) + V_i$ and $B_2 = \sqrt{(2D^2 + \beta^2) \sum_{t=t_i}^{t_{i+1} - 1} \| \bg_t \|_\star^2 }$~. 

Note that with the adopted $\beta_i$ from \cref{algo:dynamic_adaImplicit} we have that
\begin{align*}
    \frac{D^2 + \gamma C_i + \beta_i^2 }{\beta_i^2} & = \frac{D^2 + \gamma C_i + D^2 + \gamma Q_i }{D^2 + \gamma Q_i} = 2 + \gamma \frac{(C_i - Q_i)}{D^2 + \gamma Q_i} \leq 2 + \gamma \frac{\sqrt{2} D}{D^2 + \gamma \sqrt2 D 2^i},
\end{align*}
where the last inequality derives from the fact that the last term in $C_i$ which causes the algorithm to restart is such that $ \| \bx - \by \| \leq \sqrt{2} D, \, \forall \bx, \by \in \mathcal{V} $.

Therefore, we have
\begin{align*}
    R (\bu_{1:T}) & \leq \sum_{i=0}^N \sum_{t=t_i}^{t_{i+1}-1} \left( 2 + \gamma  \frac{\sqrt{2}}{D + \gamma 2^{i + \frac12}} \right) \min(B_1^i, B_2^i) \\
    & \leq \sum_{i=0}^N ( 2 + c ) \min \Bigg\{ \ell_{t_i}(\bx_{t_i}) - \ell_{t_{i+1}-1}(\bx_{t_{i+1}}) + V_i, \, \sqrt{(3 D^2 + \gamma \sqrt2 D 2^i) \sum_{t=t_i}^{t_{i+1}-1} \| \bg_t \|_\star^2}  \Bigg\} \\
    & \leq (2 + c) \sum_{i=0}^N \min \Bigg\{ \ell_{t_i}(\bx_{t_i}) - \ell_{t_{i+1}-1}(\bx_{t_{i+1}}) + V_i, \, \sqrt{(3 D^2 + \gamma C_i) \sum_{t=t_i}^{t_{i+1}-1} \| \bg_t \|_\star^2} \Bigg\} \\
    & \leq (2 + c) \min \bigg\{ \underbrace{\sum_{i=0}^N (\ell_{t_i}(\bx_{t_i}) - \ell_{t_{i+1}-1}(\bx_{t_{i+1}}) + V_i )}_{(a)}, \,\underbrace{\sum_{i=0}^N \sqrt{(3 D^2 + \gamma C_i) \sum_{t=t_i}^{t_{i+1}-1} \| \bg_t \|_\star^2}}_{(b)} \bigg\},
    \end{align*}
where in the the second inequality we used the definition of $c$. We now analyze $(a)$ and $(b)$ separately.

For $(b)$, using the Cauchy-Schwartz inequality we have that %
\begin{align*}
    \sum_{i=0}^N \sqrt{(3 D^2 + \gamma C_i) \sum_{t=t_i}^{t_{i+1}-1} \| \bg_t \|_\star^2} &\leq \sqrt{ \sum_{i=0}^N (3D^2 + \gamma C_i)} \cdot \sqrt{\sum_{i=0}^N \sum_{t=t_i}^{t_{i+1}-1} \| \bg_t \|_\star^2} \\
    & = \sqrt{3 N D^2 + \gamma C_T} \cdot \sqrt{\sum_{t=1}^T \| \bg_t \|_\star^2 } \\
    & \leq \sqrt{ \left(3 D^2 \left(\log_2 \frac{C_T}{\sqrt2 D} + 1 \right) + \gamma C_T \right) \sum_{t=1}^T \| \bg_t \|_\star^2 }~.
\end{align*}
On the other hand, for $(a)$ we have
\begin{align*}
    \sum_{i=0}^N & ( \ell_{t_i}(\bx_{t_i}) - \ell_{t_{i+1}-1}( \bx_{t_{i+1}} ) + V_i ) \leq \ell_1(\bx_1) - \ell_T( \bx_{T+1} ) + V_T ~.
\end{align*}
Therefore, putting together the results for (a) and (b), we get the stated bound.
\end{proof}

To summarize, for sequence of comparators whose path-length is observable, from \Cref{eq:diomd_doubling_trick} we have a worst-case regret bound of
\begin{equation}
    R(\bu_{1:T}) = \tilde{\mathcal{O}} \left( \min \left\{ V_T, \sqrt{T (1 + C_T)} \right\} \right)~.
\end{equation}
In light of this last result, compared to \cite{jadbabaie2015online}, our upper bound from \Cref{thm:adaimplicit_doubling_trick} strictly improves their result when optimistic predictions are not helpful.

We stress that a doubling trick is necessary for \Cref{algo:dynamic_adaImplicit}. Indeed, in order to have a fully adaptive learning rate, we should be able to tune it as a function of two quantities varying over time, namely the path-length observed and the temporal variability of the losses paid by the algorithm. While both quantities are increasing quantities over time, they also appear both at the numerator and denominator of the learning rate $\lambda_t$. However, this would result in a non-monotone sequence of learning rates, thus contradicting the assumptions in \Cref{thm:dynamic_iomd_bound}. Also, we would like to point out that to the best of our knowledge there are no existing methods in the literature which tune the learning rates with non-monotone sequences.

\section{Learning with Expert Advice}
\label{appendix:lea}

In this section we cover the application of \cref{algo:dynamic_iomd} to the setting of Learning with Expert Advice, as explained in the main paper in \cref{sec:extension}.

\learningExpertAdvice*

\begin{proof}
Given any sequence $\bu_{1:T}$, with $\bu_t \in \Delta_d$, we introduce $\bu_t' = \frac{\alpha}{d} \vec{\bm{1}} + (1-\alpha) \bu_t $, where $\vec{\bm{1}}$ is the $d$-dimensional all-ones vector. Note that $\bu_t' \in \Delta_d^\alpha$ by definition.
The regret can be decomposed as follows
\begin{align*}
    R_T(\bu_{1:T}) & = \sum_{t=1}^T \ell_t(\bx_t) - \sum_{t=1}^T \ell_t(\bu_t) = \underbrace{\sum_{t=1}^T \langle \bg_t, \bx_t - \bu_t' \rangle}_{(a)} + \underbrace{\sum_{t=1}^T \langle \bg_t, \bu_t' - \bu_t \rangle}_{(b)}~.
\end{align*}
Now, note that for $(b)$
\begin{align}
\label{eq:(b)}
    \sum_{t=1}^T \langle \bg_t, \bu_t' - \bu_t \rangle &\le \sum_{t=1}^T \|\bg_t\|_\infty \|\bu_t' - \bu_t\|_1 \nonumber\\
    & \le L_\infty \sum_{t=1}^T \left\| \frac{\alpha}{d} \vec{\bm{1}} + (1-\alpha) \bu_t - \bu_t \right\|_1 \nonumber\\
    & \le L_\infty \sum_{t=1}^T \left(\frac{\alpha}{d} \|\vec{\bm{1}}\|_1 + \alpha \|\bu_t\|_1\right) \nonumber\\
    &= 2 L_\infty T \alpha~,
\end{align}
where the second-to-last inequality derives from applying the triangle inequality.

We can now analyze $(a)$. Note that for any $\bu_t', \bu_{t-1}', \bx_t \in \mathcal{V}$ the following holds
\begin{align*}
    B_\psi (\bu_t', \bx_t) - B_\psi(\bu_{t-1}', \bx_t) &= \psi(\bu_t') - \psi(\bu_{t-1}') - \langle \nabla \psi(\bx_t), \bu_t' - \bu_{t-1}' \rangle \\
    & = - B_\psi(\bu_{t-1}', \bu_t') + \langle \nabla \psi(\bx_t) - \nabla \psi(\bu_t'), \bu_{t-1}' - \bu_t' \rangle \\
    & \leq \| \nabla \psi(\bx_t) - \nabla \psi(\bu_t') \|_\infty \| \bu_{t-1}' - \bu_t' \|_1 \\
    & \leq \ln \frac{d}{\alpha} \cdot \| \bu_{t-1}' - \bu_t' \|_1,
\end{align*}
where the last inequality derives from the fact that $ \| \nabla \psi(\bx_t) - \nabla \psi(\bu_t') \|_\infty = \max_{i \in [d]} \ln \tfrac{x_{t,i}}{u_{t,i}'} \leq \ln \tfrac{d}{\alpha} $. 

Therefore, using the prescribed learning rate and applying \cref{thm:dynamic_iomd_bound} we get
\begin{align}
\label{eq:(a)}
    \sum_{t=1}^T \langle \bg_t, \bx_t - \bu_t' \rangle \nonumber & \le \lambda_T D^2 + \gamma \sum_{t=2}^T \lambda_t \|\bu_t' - \bu_{t-1}'\|_1 + \sum_{t=1}^T \delta_t \nonumber\\
    & \le \lambda_T \left( D^2 + \ln \frac{d}{\alpha} \sum_{t=2}^T (1-\alpha)\|\bu_t - \bu_{t-1}\|_1 \right) + \beta^2 \lambda_{T+1} \nonumber\\
    & \le \lambda_{T+1} \left( \ln \frac{d}{\alpha} + \tau \ln \frac{d}{\alpha} + \beta^2  \right)~,
\end{align}
where the last inequality derives from bounding the diameter of $\Delta_d^\alpha$ with respect to the KL as done in \cref{eq:max_KL}, $1-\alpha \le 1$ and the assumption on $C_T(\bu_{1:T}) \le \tau$, while the second-to-last inequality from the definition of $\bu_t'$.

Similarly to the proof of \cref{thm:adaImplicit_regret}, we have that
\begin{align} 
\label{eq:lea_helper_1}
    \beta^2 \lambda_{T+1} & = \sum_{t=1}^T \ell_t(\bx_t) - \ell_t(\bx_{t+1}) - \lambda_t B_\psi(\bx_{t+1}, \bx_t) \nonumber \\ &\le \sum_{t=1}^T \ell_t(\bx_t) - \ell_t(\bx_{t+1}) \nonumber \\
    &\le \ell_1(\bx_1) - \ell_T(\bx_{T+1}) + \sum_{t=2}^T \max_{\bx \in \Delta_d^\alpha} \ell_t(\bx) - \ell_{t-1}(\bx) \nonumber \\
    &\le \ell_1(\bx_1) - \ell_T(\bx_{T+1}) + \sum_{t=2}^T \max_{\bx \in \Delta_d} \ell_t(\bx) - \ell_{t-1}(\bx) \nonumber\\
    &= \ell_1(\bx_1) - \ell_T(\bx_{T+1}) + V_T~.
\end{align}

On the other hand, we can improve the second part of the bound compared to a standard application of \cref{thm:adaImplicit_regret}. Indeed, we have that
\begin{equation*}
    \delta_t \le \langle \bg_t, \bx_t - \bx_{t+1} \rangle \le \langle \bg_t, \bx_t \rangle := \mathbb{E}_t[\bg_t]~,
\end{equation*}
where $\mathbb{E}_t[\bg_t]$ is the expected value of $\bg_t$ under the distribution $\bx_t$.
Furthermore, using the local norms bound for \emph{Mirror Descent} \citep[Section 6.5]{Orabona19} we have that
\begin{align*}
    \delta_t &= \langle \bg_t, \bx_t - \bx_{t+1} \rangle - \lambda_t B_\psi(\bx_{t+1}, \bx_t) \le \frac{1}{2\lambda_t}\|\bg_t\|^2_{(\nabla^2 \psi(\bz_t))^{-1}} = \frac{1}{2\lambda_t}\sum_{i=1}^d z_{t,i} g_{t,i}^2~,
\end{align*}
for a certain $\bz_t=\theta_t \bx_t + (1-\theta_t)\bx_{t+1}$ and $\theta_t \in [0, 1]$. Now, observe that from the fact that $\delta_t \ge 0$ we get $\langle \bg_t, \bx_t \rangle \ge \langle \bg_t, \bx_{t+1}\rangle$. Therefore, it follows that $ \langle \bg_t, \bz_t \rangle = \theta \langle \bg_t, \bx_t \rangle + (1-\theta)\langle \bg_t, \bx_{t+1} \rangle \le \langle \bg_t, \bx_t \rangle$, since the $g_{t,i} \ge 0$ for all $i$. Hence, we have that $\delta_t \le \frac{1}{2\lambda_t} \sum_{i=1}^d x_{t,i} g_{t,i}^2 := \frac{\mathbb{E}_t[\bg_t^2]}{2\lambda_t}$. 

To summarize, we have that
\begin{equation*}
    \lambda_{t+1} = \lambda_t + \frac{1}{\beta^2} \delta_t \le \lambda_t + \frac{1}{\beta^2} \min\left\{ \mathbb{E}_t[\bg_t], \frac{\mathbb{E}_t[\bg_t^2]}{2\lambda_t} \right\}~.
\end{equation*}
Therefore, applying \cref{lemma:adahedge} with $\Delta_t = \lambda_t, a_t^2 = \mathbb{E}_t[\bg_t^2],\, b_t=\mathbb{E}_t[\bg_t], d=c=\frac{1}{\beta^2}$ we have that
\begin{align}
\label{eq:local_norms_bound}
    \lambda_{T+1} & \le \sqrt{\frac{1}{\beta^4} \sum_{t=1}^T \left(\mathbb{E}_t[\bg_t]\right)^2 + \frac{1}{\beta^2} \sum_{t=1}^T \mathbb{E}_t[\bg_t^2]} \nonumber \\
    & \le \frac{1}{\beta^2} \sqrt{(1 + \beta^2) \sum_{t=1}^T \mathbb{E}_t[\bg_t]^2}~,
\end{align}
where the last step derives from Jensen's inequality, i.e., $(\mathbb{E}_t[\bg_t])^2 \le \mathbb{E}_t[\bg_t^2]$. 
Therefore, putting together \cref{eq:lea_helper_1} and \cref{eq:local_norms_bound} we get
\begin{align*}
    \lambda_{T+1} &\le \frac{1}{\beta^2} \min\Big\{ \ell_1(\bx_1) - \ell_T(\bx_{T+1}) + V_T,\, \sqrt{(1 + \beta^2) \textstyle\sum_{t=1}^T \mathbb{E}_t[\bg_t^2]} \Big\}~,
\end{align*}
Finally, note that from our choice of $\alpha$ and $\beta^2$, we have that
\begin{equation*}
    \frac{\ln \frac{d}{\alpha} + \tau \ln \frac{d}{\alpha} + \beta^2}{\beta^2} = \frac{(1+\tau)\ln T+\beta^2}{\beta^2} = 2~.
\end{equation*}
Hence, adding together \cref{eq:(a)} and the upper bounds to \cref{eq:(b)} in \cref{eq:lea_helper_1} and \cref{eq:local_norms_bound} yields the stated result.
\end{proof}

We next provide a corollary which shows that the bound in \cref{thm:lea} implies a first-order bound which depends on the loss of the sequence of competitors.

\begin{corollary}
Assume that $\max_{i, t} g_{t,i} = L_\infty$. Then, under the same assumptions of \cref{thm:lea}, \cref{algo:dynamic_iomd} guarantees
\begin{equation*}
    R_T(\bu_{1:T}) \le 2\sqrt{L_\infty(1+(1+\ln\tau)\ln T) L_T(\bu_{1:T})} + \mathcal{O}(\ln T)~.
\end{equation*}
\end{corollary}
\begin{proof}
Let $L_T = \sum_{t=1}^T \ell_t(\bx_t)$ and $L_T(\bu_{1:T}) = \sum_{t=1}^T \ell_t(\bu_t)$. Consider the second term in the minimum of the bound in \cref{eq:lea_regret}. We have that
\begin{align*}
    L_T - L_T(\bu_{1:T}) &\le B' \sqrt{ \sum_{t=1}^T \sum_{i=1}^d x_{t,i} g_{t,i}^2} + B'' 
    \le B' \sqrt{L_\infty L_T } + B''~,
\end{align*}
where $B'=2\sqrt{1+(1+\ln\tau)\ln T}$ and $B''=2L_\infty d$. Rearranging terms, we have that
\begin{equation*}
    L_T - \underbrace{B'\sqrt{L_\infty}}_{=b}\sqrt{L_T} - \underbrace{(L_T(\bu_{1:T}) + B'')}_{=c} \le 0
\end{equation*}
We use a result %
which says that given $x,b,c \in \mathbb{R}_+$, if $x - b\sqrt{x} - c \leq 0$ holds, then $x \leq c + b^2 + b\sqrt{c}$. Applying this result with $x=\sqrt{L_T}$ we get
\begin{equation*}
    L_T \le L_T(\bu_{1:T}) + 2L_\infty d + 4 L_\infty(1+(1+\ln\tau)\ln T) + 2\sqrt{L_\infty(1+(1+\ln\tau)\ln T)L_T(\bu_{1:T})}~.
\end{equation*}
Rearranging terms yields the stated result.
\end{proof}

\section{Composite losses} 
\label{appendix:composite}

We are now going to derive a regret bound on the case of composite losses for the static regret scenario using the algorithm \emph{AdaImplicit} from \cite{campolongo2020temporal}. 

\begin{theorem} 
\label{regret:composite_losses}
Let $\mathcal{V} \subset X \subseteq \mathbb{R}^d $ be a non-empty closed convex set. Let $\ell_t(\bx) = \tilde{\ell}_t(\bx) + r(\bx) $, where $r: X \rightarrow \mathbb{R}$ is a convex function. Let $ B_\psi $ be the Bregman divergence with respect to $\psi: X \rightarrow \mathbb{R}$. Assume $\psi$ to be 1-strongly convex w.r.t. $\|\cdot \|$ and let $\lambda_t = 1/\eta_t$. Then, \cref{algo:dynamic_iomd} with $\lambda_1=0$ and $ \lambda_t = \frac{1}{D^2} \sum_{i=1}^{t-1} \ell_i(\bx_i) - \ell_i(\bx_{i+1}) - \lambda_i B_\psi(\bx_{i+1}, \bx_i)$ for $t=2, \ldots, T$ incurs the following regret bound against any $\bu \in \mathcal{V}$
\begin{align}
        R_T(\bu) & \leq \min \Big\{ 2(\ell_1(\bx_1) - \ell_T(\bx_{T+1}) + V_T ),\, 2 D \sqrt{3 \textstyle\sum_{t=1}^T \|\bg_t \|_\star^2 } \Big\},
\end{align}
where $V_T = \sum_{t=2}^T \max_{\bx \in \mathcal{V}} \tilde\ell_t(\bx) - \tilde\ell_{t-1}(\bx)$, and $\bg_t \in \partial\ell_t(\bx_t)$.
\end{theorem}
\begin{proof}
First, let $\bg_t' \in \partial \tilde{\ell}_t(\bx_{t+1})$. Note that for any $\bu \in V$ we have the following
\begin{align*}
    \eta_t (\ell_t(\bx_{t+1}) - \ell_t(\bu)) &\leq \eta_t \langle \bg_t' + \nabla r(\bx_{t+1}), \bx_{t+1} - \bu \rangle \\
        & = \langle \eta_t \bg_t' + \nabla \psi(\bx_{t+1}) - \nabla \psi(\bx_t) +\eta_t \nabla r(\bx_{t+1}), \bx_{t+1} - \bu \rangle -\langle \nabla \psi(\bx_{t+1}) - \nabla \psi(\bx_t), \bx_{t+1} - \bu \rangle \\
        & \leq \langle \nabla \psi(\bx_{t+1}) - \nabla \psi(\bx_t), \bu - \bx_{t+1} \rangle \\
        & = B_\psi(\bu, \bx_t) - B_\psi(\bu, \bx_{t+1}) - B_\psi(\bx_{t+1}, \bx_t),
\end{align*}
where the second inequality derives from the optimality condition of the update rule.

Remember that $\delta_t = \ell_t(\bx_t) - \ell_t(\bx_{t+1}) - \frac{B_\psi(\bx_{t+1}, \bx_t)}{\eta_t} $. After adding $\ell_t(\bx_t)$ on both sides, taking $\ell_t(\bx_{t+1})$, dividing both sides by $\eta_t$ and summing over time we get
\begin{align*}
    \sum_{t=1}^T \left( \ell_t(\bx_{t}) - \ell_t(\bu) \right)& \leq \sum_{t=1}^T \frac{B_\psi(\bu, \bx_t) - B_\psi(\bu, \bx_{t+1})}{\eta_t} + \sum_{t=1}^T \delta_t \\
    & \leq \frac{D^2}{\eta_1} + D^2 \sum_{t=2}^T \left( \frac{1}{\eta_t} - \frac{1}{\eta_{t-1}} \right) + \sum_{t=1}^T \delta_t \\
    & \le 2 D^2 \lambda_{T+1}~.
\end{align*}
Now, define $\tilde\delta_t = \tilde\ell_t(\bx_t) - \tilde\ell_t(\bx_{t+1}) - \lambda_t B_\psi(\bx_{t+1}, \bx_t)$. Note that
\begin{align*}
    \sum_{t=1}^T \delta_t &= \sum_{t=1}^T \left[\tilde\delta_t + r(\bx_t) - r(\bx_{t+1}) \right] \\
        & \le \sum_{t=1}^T \left(\tilde\ell_t(\bx_t) - \tilde\ell_t(\bx_{t+1}) \right) + r(\bx_1) - r(\bx_{T+1}) \\
        & = \ell_1(\bx_1) - \ell_{T}(\bx_{T+1}) + \sum_{t=2}^T \left(\tilde\ell_t(\bx_t) - \tilde\ell_{t-1}(\bx_t)\right) \\
        & \le \ell_1(\bx_1) - \ell_T(\bx_{T+1}) + \sum_{t=1}^T \max_{\bx \in \mathcal{V}} \tilde\ell_t(\bx) - \tilde\ell_{t-1}(\bx)~.
\end{align*}
The rest of the proof follows from the proof of \cref{thm:adaImplicit_regret} in \cref{appendix:technical}.
\end{proof}
Compared to \citet{song2018fully}, the above regret bound is adaptive to the gradients of the loss function. Furthermore, the regret bound from \citet{song2018fully} does not contain the temporal variability $V_T$, which could potentially lead to constant regret if the loss function stays fixed over time. Next, we are going to show how one can adapt the previous theorem to the dynamic case.

\dynamicComposite*

\begin{proof}%
To prove the stated bound, we can adapt the proof from \cref{regret:composite_losses}. In particular, from the update rule using \cref{eq:update_optimality} we have that 
\begin{align*}
    \eta_t ( \ell_t(\bx_{t+1}) - \ell_t(\bu_t)) \leq B_\psi(\bu_t, \bx_t) - B_\psi(\bu_t, \bx_{t+1}) - B_\psi(\bx_{t+1}, \bx_t),
\end{align*}
where $\bg_t' \in \partial \ell_t(\bx_{t+1})$. Following the proof of \cref{thm:adaImplicit_regret}, summing $\ell_t(\bx_t)$ on both sides and rearranging terms we get
\begin{align*}
    R_T (\bu_{1:T}) 
    &\leq \sum_{t=1}^T \lambda_t (B_\psi(\bu_t, \bx_t) - B_\psi(\bu_t, \bx_{t+1})) + \sum_{t=1}^T \left[ \tilde\delta_t + \beta(r(\bx_t) - r(\bx_{t+1})) \right] \\
    & \leq 2 (D^2 + \gamma C_T + \beta^2) \lambda_{T+1}~.
\end{align*}
From the last inequality, substituting the value of $\beta^2$ and following the proof of \cref{regret:composite_losses} yields the desired result.
\end{proof}

\section{Combining algorithms} 
\label{appendix:combiner}

\begin{algorithm}[t]
\caption{Adapt-ML-Prod}
\label{algo:adaptmlprod}
\begin{algorithmic}[1]
{
\REQUIRE{A rule to sequentially pick the learning rates, vector $\bw_0=(w_{0,1},\dots,w_{0,d})$ of nonnegative weights that sum to 1.}
\FOR{$t=1,\dots,T$}
\STATE{Pick the learning rates $\eta_{t-1,i}$ according to the rule.}
\STATE{Define $\bp_t$ such that $p_{t,i}=\eta_{t-1,i}w_{t-1,i} / \bm{\eta}_{t-1}^\top \bw_{t-1}$}
\STATE{Observe $\bg_t$ and incur loss $\ell_t(\bp_t) = \langle \bg_t, \bp_t \rangle$}
\STATE{For each expert $i$ perform the update
$$ w_{t,i} = \left( w_{t-1,}\left( 1 + \eta_{t-1,i}(\ell_t(\bp_t) - \ell_{t,i}) \right) \right)^{\frac{\eta_{t,i}}{\eta_{t-1,i}}}$$
}
\ENDFOR
}
\end{algorithmic}
\end{algorithm}

In this section, we give in detail the results related to \cref{sec:combiner}.

First, we point out that \cref{algo:prod} is an application of the more general \emph{Adapt-ML-Prod} algorithm, which is given in \cref{algo:adaptmlprod}. We recall the following theorem which provides a regret bound to \cref{algo:adaptmlprod} (proof omitted).
\begin{theorem}{\citep[Theorem 3]{gaillard2014second}}
\label{thm:adaptmlprod}
For all sequences of loss vectors $\bg_t \in [0,1]^d$, for all rules prescribing sequences of learning rates $\eta_{t,i}$ that, for each $i$, are non-increasing in $t$, \cref{algo:adaptmlprod} ensures
\begin{equation*}
    R_T(\be_i) \le \frac{1}{\eta_{0,i}}\ln\frac{1}{w_{0,i}}+\sum_{t=1}^T \eta_{t-1,i} r_{t,i}^2 + \frac{1}{\eta_{T,i}} \ln K_T~,
\end{equation*}
where $r_{t,i} = \langle \bg_t, \bp_t \rangle - g_{t,i}$ and $K_T = 1 + \frac{1}{e} \sum_{t=1}^T \sum_{j=1}^d \left(\frac{\eta_{t-1,j}}{\eta_{t,j}} - 1 \right)$.
\end{theorem}

\subsection{Strongly-adaptive algorithms}

Strongly adaptive algorithms enjoy a $\tilde{\mathcal{O}}(\sqrt{I})$ regret bound for any interval $I=[s, e] \subseteq [1,T]$.
The argument used to prove the optimal dynamic regret bound for strongly adaptive algorithms is reported next for completeness. We first provide a lemma that we will use in the proof of the main result.

\begin{lemma}{\citep{cutkosky2020parameter}}
\label{lemma:breaking_interval}
Consider a set $\mathcal{V} \subset \mathbb{R}^d$ such that $\max_{\bx, \by \in \mathcal{V}} \| \bx - \by \| \leq D$. Define a time interval $I=[s,e]$ and let $\bu_t \in \mathcal{V}$ for any $t \in I$. Let $C_I = \sum_{t=s+1}^e \| \bu_t - \bu_{t-1}\| $. Then it is possible to break the interval $I$ into $K$ disjoint intervals $ I = J_1 \cup \cdots \cup J_K$ such that for each $i$ we have that $C_{J_i} \leq 2 D$~.
Furthermore, we have that $K \leq \frac{C_I + D}{D}$.
\end{lemma}
\begin{proof}
We can build the set of subintervals $J_1, \ldots, J_K$ iteratively. Define $J_1=[s, t_1]$ as the interval such that $t_1$ is the first time-step when $\sum_{t=s+1}^{t_1} \| \bu_t - \bu_{t-1} \| \geq D$. Then we have that $C_{J_1} \leq 2 D$ (since $\|\bu_t - \bu_{t-1}\| \leq D$ in any step $t$). We can repeat this process: given $t_{i-1}$, let $t_i$ be the time-step such that $C_{[t_{i-1}, t_i]} \geq D$ and define $J_i = [t_{i-1}, t_i]$. If such a $t_i$ does not exist then set $i=K$ and $t_i = e$. Then for any subinterval we have that $C_{J_i} \leq 2 D$. 

On the other hand, we have $C_{J_i} \ge D$ for all $i$ but the last one. Assume that there are $K$ subintervals. We have that $\sum_{i=1}^K C_{J_i} \leq C_{I}$. Therefore,
\begin{equation*}
    C_I \ge \sum_{i=1}^K C_{J_i} \ge (K-1) D~,
\end{equation*}
from which the desired result follows.
\end{proof}

We can now prove that result regarding the dynamic regret of strongly adaptive algorithms satisfying a certain condition on the path-length, as stated in \cref{thm:strongly_adaptive_path_length} (restated here for completeness).
\combinerDynamic*
\begin{proof}
Note that for any $\bx, \by \in \mathcal{V}$ we have $\frac12\|\bx-\by\|^2 \le B_\psi(\bx, \by) \le D^2$. Hence $\|\bx-\by\|\le \sqrt2D$.
Let $J_1, \ldots, J_K$ be the set of disjoint intervals resulting from the construction in \Cref{lemma:breaking_interval}. We have that $J = J_1 \cup \dots  \cup J_K$, such that for each $i$ we have $C_{J_i} \leq 2\sqrt2 D$, and $K \leq \frac{C_I + \sqrt2D}{\sqrt2D}$.
Now, on each of intervals $J_i = [s_i, e_i]$, using \Cref{eq:cutkosky_condition_1} the regret is bounded as
\begin{align*}
    R(\bu_{s_i, e_i}) &\leq \mathcal{O}\left( (D + C_{J_i}) \sqrt{|J_i|} \right) = \mathcal{O}\left( D \sqrt{|J_i|} \right)~,
\end{align*}
where in the last inequality we used the fact that $C_{J_i} \leq 2\sqrt2 D$ from \cref{lemma:breaking_interval}.

Then, for interval $J$ using the Cauchy-Schwartz inequality we have that 
\begin{align*}
    R(\bu_{s:e}) &= \sum_{i=1}^K R(\bu_{s_i:e_i}) \leq \sum_{i=1}^K \mathcal{O}\left(D\sqrt{|J_i|} \right) \le \mathcal{O}\left(D \sqrt{K \sum_{i=1}^K |J_i|}\right) \le \mathcal{O}\left(D \sqrt{\frac{C_I + D}{D}|J|} \right) \\
    & = \mathcal{O} \left( \sqrt{|J|D(C_I + D)} \right)
\end{align*}
where in the third inequality we used the fact that $K \leq \frac{C_I + \sqrt2D}{\sqrt2D}$ from \cref{lemma:breaking_interval}.
\end{proof}

In particular, note that the above theorem is valid also with $J =[1,T]$. 

We next prove \cref{thm:combiner}, providing a regret bound to the algorithm given in \cref{sec:combiner}.

\combinerTheorem*
\begin{proof}
Let $\by_{t}^i$ be the output of algorithm $i \in \{\mathcal{A,B}\}$ at time $t$. For any $i \in \{\mathcal{A,B}\}$, we can decompose the regret as follows
\begin{align}
\label{eq:prod_decomposition}
    \sum_{t=1}^T (\ell_t(\bx_t) - \ell_t(\bu_t)) &= \sum_{t=1}^T (\ell_t(\bx_t) - \ell_{t}(\by_{t}^i)) + \sum_{t=1}^T (\ell_{t}(\by_{t}^i) - \ell_t(\bu_{1:T})) \nonumber\\
    & \le \sum_{t=1}^T \sum_{i \in \{\mathcal{A, B}\}} p_{t,i} \ell_{t}(\by_t^i) -  \sum_{t=1}^T\ell_t(\by_t^i) + R_{T,i}(\bu_{1:T})~,
\end{align}
where we used Jensen's inequality and the regret guarantee of algorithm $i$ in the last step.

We now analyze the first part in the above bound. Recall that $w_{1,\mathcal{A}} = w_{1,\mathcal{B}} = \frac12$. For algorithm $\mathcal{A}$ we denote $\by_t^\mathcal{A} = \ba_t$. Also, we assume $\eta_t = \eta_{t, \mathcal{A}}$ and $\eta_{t, \mathcal{B}} = 1/2$ (cf. \cref{algo:prod}).

From \cref{thm:adaptmlprod} we have that 
\begin{align*} 
    \sum_{t=1}^T \sum_{i \in \{\mathcal{A}, \mathcal{B} \}} p_{t,i} \ell_{t}(\by_t^i) - \sum_{t=1}^T \ell_t(\ba_t) & \le \frac{1}{\eta_{0,\mathcal{A}}} \ln \frac{1}{w_{0,\mathcal{A}}} + \sum_{t=1}^T \eta_{t-1, \mathcal{A}} r_{t,\mathcal{A}}^2 + \frac{1}{\eta_{T,\mathcal{A}}} \ln K_T \\
    & = 2 \ln 2 + \sum_{t=1}^T \eta_{t-1,\mathcal{A}}(\ell_t(\bb_t) - \ell_t(\ba_t))^2 + \frac{1}{\eta_{T,\mathcal{A}}} \ln K_T \\
    & \le 2\ln2 + \sum_{t=1}^T \frac{(\ell_t(\bb_t) - \ell_t(\ba_t))^2}{\sqrt{1 + \sum_{i=1}^{t-1} (\ell_t(\bb_i) - \ell_t(\ba_i))^2}} + \ln K_T \sqrt{1 + \sum_{t=1}^T (\ell_t(\bb_t) - \ell_t(\ba_t))^2} \\
    & \leq 2\ln2 + (2 + \ln K_T)\sqrt{1 + \sum_{t=1}^T (\ell_t(\bb_t) - \ell_t(\ba_t))^2} \\
    & \le 2\ln2 + (2 + \ln K_T) \sqrt{T + 1}~,
\end{align*}
where the second-to-last inequality derives from applying \citet[Lemma 4.13]{Orabona19}, while the last one from the fact that the losses are bounded in $[0,1]$. Note that $\ln K_T = \mathcal{O}(\ln \ln T)$ from \citet[Corollary 4]{gaillard2014second}. Therefore, using this last result in \cref{eq:prod_decomposition} and the fact that $R_{T, \mathcal{A}} \le $ by assumption we have that 
\begin{equation}
\label{eq:regret_A}
    R_T(\bu_{1:T}) \le \mathcal{O}\left(\ln K_T \sqrt{T+1} + \sqrt{T D(C_T + D)} \right) = \tilde{\mathcal{O}}\left(\sqrt{TD(C_T + D)}\right)~.
\end{equation}

On the other hand, for algorithm $\mathcal{B}$ by applying again \cref{thm:adaptmlprod} we have that
\begin{align*}
    \sum_{t=1}^T \sum_{i \in \{\mathcal{A,B}\}} p_{t,i} \ell_t(\by_t^i) - \sum_{t=1}^T \ell_t(\bb_t) &\le \frac{1}{\eta_{0,\mathcal{B}}} \ln \frac{1}{w_{0,\mathcal{B}}} + \sum_{t=1}^T \eta_{t-1,\mathcal{B}} r_{t,\mathcal{B}}^2 + \frac{1}{\eta_{T,\mathcal{B}}} \ln K_T \\
    & = 2\ln2 + 2\ln K_T~,
\end{align*}
since $\eta_t = \frac12$ and $r_{t,\mathcal{A}} = 0$ for all $t$ by assumption (see \cref{algo:prod}). Using this last result in \cref{eq:prod_decomposition} we get
\begin{equation}
\label{eq:regret_B}
    R_T(\bu_{1:T}) \le \ell_1(\bx_1) - \ell_T(\bx_{T+1}) + V_T + 2\ln2 + 2\ln K_T = \tilde{\mathcal{O}}\left( V_T \right)~.
\end{equation}
Taking the minimum between \cref{eq:regret_A} and \cref{eq:regret_B} concludes the proof.
\end{proof}

\subsection{Learning with Expert Advice}

\begin{algorithm}[t]
\caption{Generic strongly-adaptive algorithm}
\label{algo:strongly_adaptive}
\begin{algorithmic}[1]
{
\REQUIRE{Non-empty closed convex set $\mathcal{V} \subset \mathbb{R}^d$, OLO algorithm $\mathcal{A}$, expert algorithm $\mathcal{B}$}
\FOR{$t=1,\dots,T$}
\STATE{Receive predictions from $\mathcal{A}_1, \ldots, \mathcal{A}_t$, denoted $\by_t^1, \ldots, \by_t^t$}
\STATE{Receive distribution $\bp_t$ from $\mathcal{B}$}
\STATE{Output $\bx_t = \sum_{i=1}^t p_{t,i} \by_t^i$}
\STATE{Observe loss $\ell_t$ and pay $\ell_t(\bx_t)$}
\STATE{Pass $\ell_t$ to $\mathcal{A}_1, \ldots, \mathcal{A}_t$}
\STATE{Set $g_{t,i} = \ell_t(\by_t^i)$ and pass $\bg_t$ to $\mathcal{B}$}
\ENDFOR
}
\end{algorithmic}
\end{algorithm}

In this section we sketch the strategy to get the optimal bound for the setting of Learning with Expert Advice (LEA) by combining the prediction of a strongly-adaptive algorithm and \cref{algo:greedy}, as suggested by \cref{thm:combiner}. In particular, we need to design a strongly-adaptive algorithm which satisfies the conditions of \cref{thm:strongly_adaptive_path_length}.

\paragraph{Strongly-adaptive algorithm.} The dominant approach in the design of strongly adaptive algorithms has been the following. Consider an anytime algorithm $\mathcal{A}$ with static regret bound of $\tilde{\mathcal{O}}(\sqrt{t})$ for the interval $[1, t]$. At each time-step $t$ initialize a new copy of $\mathcal{A}$. Then, to come up with a prediction at round $t$, use an expert algorithm $\mathcal{B}$ to combine the predictions of the $t$ existing base learners. The resulting strategy is depicted in \Cref{algo:strongly_adaptive}.
Let $\by_t^s$ the output at time $t$ of the algorithm initialized at time $s$, i.e., $\mathcal{A}_s$. The regret over an interval $I = [s,e]$ against any sequence $\bu_{s:e} $ can then be decomposed as follows
\begin{align}
\label{eq:strongly_adaptive_bound}
    \sum_{t=s}^e (\ell_t(\bx_t) - \ell_t(\bu_t)) &\leq \sum_{t=s}^e \sum_{i=1}^t p_{t,i} \ell_t(\by_t^i) - \sum_{t=1}^T \ell_t(\bu_t) \nonumber \\
    & = \underbrace{\sum_{t=s}^e ( \langle \bp_t, \bg_t \rangle - \ell_t(\by_t^s) )}_{\textup{Regret of } \mathcal{B}} + \underbrace{\sum_{t=s}^e (\ell_t(\by_t^s) - \ell_t(\bu_t))}_{\textup{Regret of } \mathcal{A}_s}~.
\end{align} 
where the first inequality derives from Jensen's inequality. Next, we analyze the two contributions separately.

\paragraph{Algorithm $\mathcal{B}$.} This is the regret of an expert algorithm against a fictitious adversary which always commits to the same choice $\by_t^s$. In order to have the desired regret bound of $\tilde{\mathcal{O}}(\sqrt{|I|})$ a regular expert algorithm with regret bound of $\tilde{\mathcal{O}}(\sqrt{T})$ does not work. Indeed, we need an expert algorithm which only pays for the timesteps when the base algorithm $\mathcal{A}_s$ has been active, i.e., for the interval $[s, e]$. It can be shown that a \emph{sleeping} experts algorithm suffices (details omitted). Hence, we adopt the algorithm \emph{Sleeping CB} from \citet{jun2017improved} as the expert algorithm $\mathcal{B}$ in \cref{algo:strongly_adaptive}, which gives
\begin{equation}
\label{eq:sleeping}
    \sum_{t=s}^e \langle \bp_t, \bg_t \rangle - \sum_{t=s}^e \ell_t(\by_t^s) \le \tilde{\mathcal{O}}\left(\sqrt{|I|}\right)~,
\end{equation}
for any interval $I=[s,e] \subseteq [1, T]$.

\paragraph{Algorithm $\mathcal{A}$.} Differently from \citet{jun2017improved} for the case of Learning with Expert Advice, we cannot use the \emph{Coin Betting} algorithm as base algorithm (i.e., $\mathcal{A}$ in \cref{algo:strongly_adaptive}), since it does not satisfy a regret bound as the one required by \cref{eq:cutkosky_condition_1} involving the path-length of the comparator sequence. Instead, we can adopt the algorithm from \cref{thm:lea}. Indeed, by using $\beta^2 = \ln T$ in \cref{thm:lea} we get the following regret bound against any sequence $\bu_{1:T}$ in the simplex
\begin{equation}
\label{eq:adaimplicit_experts}
    R_T(\bu_{1:T}) \le (2 + C_T) \sqrt{(1+\ln T)\sum_{t=1}^T \mathbb{E}[\bg_t^2]} + 2L_\infty d = \mathcal{O}\left(C_T\sqrt{T \ln T}\right)~.
\end{equation}

\paragraph{Regret bound.} For any interval $I \subseteq [1, T]$, a strongly-adaptive algorithm run with \emph{Coin Betting} as $\mathcal{B}$ and the algorithm from \cref{thm:lea} as base algorithm $\mathcal{A}$ would get the regret bound required in \cref{eq:cutkosky_condition_1}. Indeed, from the decomposition in \cref{eq:strongly_adaptive_bound} we have that for any interval $I = [s,e]$ and sequence $\bu_{s:e}$
\begin{align*}
    \sum_{t=s}^e (\ell_t(\bx_t) - \ell_t(\bu_t)) &\le \sum_{t=s}^e ( \langle \bp_t, \bg_t \rangle - \ell_t(\by_t^s) ) + \sum_{t=s}^e (\ell_t(\by_t^s) - \ell_t(\bu_t)) \\
    & \le \tilde{\mathcal{O}}\left(\sqrt{|I|}\right) + \mathcal{O}\left( C_{I} \sqrt{|I| \ln T} \right) \\
    &= \tilde{\mathcal{O}}\left(C_I\sqrt{|I|}\right)~,
\end{align*}
where the first inequality derives from the regret guarantee of the sleeping expert algorithm in \cref{eq:sleeping} and the one of the base algorithm in \cref{eq:adaimplicit_experts}.
We can therefore apply \cref{thm:strongly_adaptive_path_length} with $J= [1,T]$.
This immediately gives a dynamic regret bound of $\mathcal{O}\left(\sqrt{T C_T \ln T} \right)$ for the resulting strongly-adaptive algorithm against any sequence $\bu_{1:T}$. To conclude, we can combine the predictions of the strongly-adaptive algorithm and those from the greedy strategy with \cref{algo:prod} and apply \cref{thm:combiner} to get the desired result.

\paragraph{Running Time.} Note that the tecnique sketched in \cref{algo:strongly_adaptive} requires initializing a new algorithm in any step, which would lead to a total runtime of $\mathcal{O}(T^2)$. However, there are techniques to reduce this running time to $\mathcal{O}(T\ln T)$ such as \emph{Geometric Covering} intervals \citep{jun2017improved}.

\end{document}